\documentclass{article} 
\usepackage{iclr2023_conference,times}

\usepackage{amsmath,amsfonts,bm}

\def\eqref#1{(\ref{#1})}

\def\1{\bm{1}}

\def\rd{{\textnormal{d}}}

\def\vh{{\bm{h}}}

\def\vp{{\bm{p}}}
\def\vq{{\bm{q}}}

\def\vu{{\bm{u}}}
\def\vv{{\bm{v}}}
\def\vw{{\bm{w}}}

\DeclareMathAlphabet{\mathsfit}{\encodingdefault}{\sfdefault}{m}{sl}
\SetMathAlphabet{\mathsfit}{bold}{\encodingdefault}{\sfdefault}{bx}{n}

\newcommand{\R}{\mathbb{R}}

\usepackage{hyperref}
\usepackage{url}
\usepackage{algorithm}
\usepackage{algorithmic}

\usepackage{amsmath}
\usepackage{amsthm}
\usepackage{amssymb}
\usepackage{booktabs}
\usepackage{multirow}
\usepackage{enumitem}
\usepackage{bm}
\usepackage{caption}
\usepackage{graphicx}
\usepackage{wrapfig}

\usepackage{array}
\newcommand{\PreserveBackslash}[1]{\let\temp=\\#1\let\\=\temp}
\newcolumntype{C}[1]{>{\PreserveBackslash\centering}p{#1}}
\newcolumntype{R}[1]{>{\PreserveBackslash\raggedleft}p{#1}}
\newcolumntype{L}[1]{>{\PreserveBackslash\raggedright}p{#1}}

\newcommand{\std}[1]{{\,\fontsize{6pt}{7pt}\selectfont{\scalebox{0.8}[1.0]{$\pm$}#1}}}
\newcommand{\zz}{{\textcolor{white}{0}}}

\newcommand{\dif}{{\mathrm{d}}}
\newcommand{\dt}{{\mathrm{d}t}}
\newcommand{\vV}{{\bm{V}}}
\newcommand{\vlambda}{{\bm{\lambda}}}

\newcommand{\M}{{\mathcal{M}}}

\newcommand{\TuM}{{\mathcal{T}_\vu\mathcal{M}}}
\newcommand{\TcuM}{{\mathcal{T}^*_\vu\mathcal{M}}}
\newcommand{\TvuM}{{\mathcal{T}_{(\vv,\vu)}\mathcal{M}}}
\newcommand{\TuuM}{{\mathcal{T}_{(\vu^{s+1},\vu^{s})}\mathcal{M}}}
\newcommand{\ddt}[1][]{\frac{\mathrm{d}{#1}}{\dt}}

\newcommand{\pderiv}[2]{{\frac{\partial #1}{\partial #2}}}

\newcommand{\bnabla}{{\overline{\nabla}}}

\newcommand{\bM}{{\overline{M}}}
\newcommand{\bY}{{\overline{Y}}}

\newcommand{\ul}[1]{{\underline{#1}}}
\newcommand{\ub}[1]{{\underline{\textbf{#1}}}}

\newenvironment{psmallmatrix}{\left(\begin{smallmatrix}}{\end{smallmatrix}\right)}

\newtheorem{thm}{Theorem}

\newtheorem{dfn}{Definition}

\title{FINDE: Neural Differential Equations for\\ Finding and Preserving Invariant Quantities}

\author{Takashi Matsubara\\
Osaka University\\
Toyonaka, Osaka, 560--8531 Japan\\
\texttt{matsubara@sys.es.osaka-u.ac.jp} \\
\And
Takaharu Yaguchi \\
Kobe University \\
Kobe, Hyogo, 657--8501 Japan\\
\texttt{yaguchi@pearl.kobe-u.ac.jp} \\
}

\iclrfinalcopy 
\begin{document}

\maketitle

\begin{abstract}
  Many real-world dynamical systems are associated with first integrals (a.k.a.~invariant quantities), which are quantities that remain unchanged over time.
  The discovery and understanding of first integrals are fundamental and important topics both in the natural sciences and in industrial applications.
  First integrals arise from the conservation laws of system energy, momentum, and mass, and from constraints on states; these are typically related to specific geometric structures of the governing equations.
  Existing neural networks designed to ensure such first integrals have shown excellent accuracy in modeling from data.
  However, these models incorporate the underlying structures, and in most situations where neural networks learn unknown systems, these structures are also unknown.
  This limitation needs to be overcome for scientific discovery and modeling of unknown systems.
  To this end, we propose \emph{first integral-preserving neural differential equation (FINDE)}.
  By leveraging the projection method and the discrete gradient method, FINDE finds and preserves first integrals from data, even in the absence of prior knowledge about underlying structures.
  Experimental results demonstrate that FINDE can predict future states of target systems much longer and find various quantities consistent with well-known first integrals in a unified manner.
\end{abstract}

\section{Introduction}
Modeling and predicting real-world systems are fundamental aspects of understanding the world in natural science and improving computer simulations in industry.
Target systems include chemical dynamics for discovering new drugs~\citep{Raff2012}, climate dynamics for climate change prediction and weather forecasting~\citep{Rasp2020,Trigo1999}, and physical dynamics of vehicles and robots for optimal control~\citep{Nelles2001}.
In addition to image processing and natural language processing~\citep{Devlin2018,He2015a}, neural networks have been actively studied for modeling dynamical systems~\citep{Nelles2001}.
Their history dates back to at least the 1990s (see \citet{Chen1990,Clouse1997,Levin1995,Narendra1990,Sjoberg1994,Wang1998} for examples).
Recently, two notable but distinct families have been proposed.
Physics-informed neural networks (PINNs) directly solve partial differential equations (PDEs) given as symbolic equations~\citep{Raissi2019}.
Neural ordinary differential equations (NODEs) learn ordinary differential equations (ODEs) from observed data and solve them using numerical integrators~\citep{Chen2018e}.
Our focus this time is on NODEs.

Most real-world systems are associated with \emph{first integrals} (a.k.a.~invariant quantities), which are quantities that remain unchanged over time~\citep{Hairer2006}.
First integrals arise from intrinsic geometric structures of systems and are sometimes more important than superficial dynamics in understanding systems (see Appendix~\ref{appendix:system} for details).
Many previous studies have extended NODEs by incorporating prior knowledge about first integrals and attempted to accurately learn a target system.
\citet{Greydanus2019} proposed the Hamiltonian neural network (HNN), which employs a neural network to approximate Hamilton's equation, thereby conserving the system energy called the Hamiltonian.
\citet{Finzi2020} proposed neural network architectures that conserve linear and angular momenta by utilizing the graph structure.
\citet{Finzi2020b} also extended an HNN to a system with holonomic constraints, which led to first integrals such as a pendulum length.
\citet{Matsubara2020} proposed a model that preserves the total mass of a discretized PDE.
These studies have demonstrated that the more prior knowledge a neural network has about first integrals, the more accurate their dynamics prediction.
See Table~\ref{tab:comparison} for comparisons.

Previous studies have mainly attempted to preserve known first integrals for better computer simulations.
However, in situations where a neural network learns a target system, it is naturally expected that first integrals associated with the target system are unknown, and it is not clear which of the above methods are available.
Therefore, this study proposes \emph{first integral-preserving neural differential equation} (FINDE) to find and preserve unknown first integrals from data in a unified manner.
FINDE has two versions for continuous and discrete time; these have the following advantages.

\vspace*{-0.5mm}\paragraph{Finding First Integrals}
Many studies have designed architectures or operations of neural networks to model continuous-time dynamics with known types of first integrals.
However, the underlying geometric structures of a target system are generally unknown in practice.
In contrast, FINDE finds various types of first integrals from data in a unified manner and preserves them in predictions.
For example, from an energy-dissipating system, FINDE can find first integrals other than energy.
FINDE can find not only known first integrals, but also unknown ones.
Hence, FINDE can lead to scientific discoveries.

\vspace*{-0.5mm}\paragraph{Combination with Known First Integrals}
FINDE can be combined with previously proposed neural networks designed to preserve known first integrals, such as HNNs.
In addition, when some first integrals are known in advance, they can also be incorporated into FINDE to avoid rediscovery.
Therefore, FINDE is available in various situations.

\vspace*{-0.5mm}\paragraph{Exact Preservation of First Integrals}
The first integral associated with a continuous-time system is destroyed after the dynamics is temporally discretized for computer simulations.
By leveraging the discrete gradient, the discrete-time version of FINDE preserves first integrals exactly (up to rounding errors) in discrete time and further improves the prediction performance.

\begin{table}[t]
  \scriptsize
  \caption{Comparison of Related Studies on Preservation of First Integrals.}
  \label{tab:comparison}
  \centering
  \begin{tabular}{lC{8mm}C{8mm}C{8mm}C{8mm}C{8mm}C{8mm}C{12mm}}
    \toprule
                                 & \rotatebox{12}{\ Energy} & \rotatebox{12}{\ Monentum} & \rotatebox{12}{\ Mass} & \rotatebox{12}{\ Constraint} & \rotatebox{12}{\ Learning invariants} & \rotatebox{12}{\ Exact conservation} & \\[-1mm]
    \midrule
    NODE~\citep{Chen2018e}       &                          &                            &                        &                              &                                       &                                        \\
    HNN~\citep{Greydanus2019}    & \checkmark               &                            &                        &                              &                                       &                                        \\
    LieConv~\citep{Finzi2020}    & \checkmark               & \checkmark                 &                        &                              &                                       &                                        \\
    DGNet~\citep{Matsubara2020}  & \checkmark               &                            & \checkmark             &                              &                                       & \checkmark                             \\
    CHNN~\citep{Finzi2020b}      & \checkmark               &                            &                        & \checkmark                   &                                       &                                        \\
    NPM~\citep{Yang2020a}        &                          &                            &                        & \checkmark                   & \checkmark                            & \checkmark                             \\
    \midrule
    Continuous FINDE  (proposed) & \checkmark               & \checkmark                 & \checkmark             & \checkmark                   & \checkmark                            &                                        \\
    Discrete FINDE   (proposed)  & \checkmark               & \checkmark                 & \checkmark             & \checkmark                   & \checkmark                            & \checkmark                             \\
    \bottomrule
  \end{tabular}
  \vspace*{-3mm}
\end{table}

\section{Background and Related Work}\label{sec:background}
\paragraph{First Integrals}\label{sec:background_first_integral}
Let us consider a time-invariant differential system $\ddt\vu=f(\vu)$ on an $N$-dimensional manifold $\M$, where $\vu$ denotes the system state and $f:\M\rightarrow\TuM$ represents a vector field on $\M$.
For simplicity, we suppose the manifold $\M$ to be a Euclidean space $\R^N$.
\begin{dfn}[first integral]
  A quantity $V:\M\rightarrow\R$ is referred to as a first integral of a system $\ddt\vu=f(\vu)$ if it remains constant along with any solution $\vu(t)$, i.e., $\ddt V(\vu)=0$.
\end{dfn}
If a differential system $\ddt\vu=f(\vu)$ has $K$ functionally independent first integrals $V_1,\dots,V_K$, the solution $\vu(t)$ given an initial value $\vu_0$ stays at the $(N-K)$-dimensional submanifold
\begin{equation}
  \M'=\{\vu\in\M:V_1(\vu)=V_1(\vu_0),\dots,V_K(\vu)=V_K(\vu_0)\}. \label{eq:submanifold}
\end{equation}
The tangent space $\TuM'\subset\TuM$ of the submanifold $\M'\subset\M$ at a point $\vu$ is the orthogonal complement to the space spanned by the gradients $\nabla V_k(\vu)$ of the first integrals $V_k$ for $k=1,\dots,K$;
\begin{equation}
  \TuM'=\{\vw\in\TuM: \nabla V_k(\vu)^\top\vw=0 \mbox{ for }k=1,\dots,K\}.
\end{equation}
Conversely, if the time-derivative $f$ at point $\vu$ is on the tangent space $\TuM'$ for certain functions $V_k$'s, the quantities $V_k$'s are first integrals of the system $\ddt\vu=f(\vu)$; it holds that $\ddt V_k(\vu)=\nabla V_k(\vu)^\top\ddt\vu=\nabla V_k(\vu)^\top f(\vu)=0.$

One of the most well-known first integrals is the Hamiltonian $H$, which represents the system energy of a Hamiltonian system.
Noether's theorem states that a continuous symmetry of a system leads to a conservation law (and hence a first integral)~\citep{Hairer2006}.
A Hamiltonian system is symmetric to translation in time, and the corresponding first integral is the Hamiltonian.
Symmetries to translation and rotation in space lead to the conservation of linear and angular momenta.
However, not all first integrals are related to symmetries.
A pendulum can be expressed in Cartesian coordinates, and then the rod length constrains the mass position.
This type of constraint is called a holonomic constraint and leads to first integrals.
Models of disease spreads and chemical reactions have the total mass (population) as the first integral.
Also for a system described by a PDE, the total mass is sometimes a first integral~\citep{Furihata2010}.
See Appendix \ref{appendix:system} for the classes of dynamics, their geometric structures, and related studies to find or preserve first integrals.

\paragraph{First Integrals in Numerical Analysis}
For computer simulations, differential systems are discretized in time and solved by numerical integration, causing numerical errors (which is composed of temporal discretization errors and rounding errors).
Moreover, the geometric structures of the system are often destroyed, and the corresponding first integrals are no longer preserved.
A common remedy is a symplectic integrator, which preserves the symplectic structure and accurately integrates Hamiltonian systems~\citep{Hairer2006}.
However, the Ge--Marsden theorem states that a symplectic integrator only approximately conserves the Hamiltonian~\citep{Zhong1988}.
Hence, many numerical schemes have also been investigated to preserve first integrals exactly, while these schemes cannot preserve the symplectic structure.
Some examples are shown below.

Let the superscript $s\in\{0,1,\dots,S\}$ denote the state $\vu^s$ or time $t^s$ at the $s$-th time step, and $\Delta t^s=t^{s+1}-t^{s}$ denote a time-step size.
A projection method uses a numerical integrator to predict the next state $\tilde\vu^{s+1}$ from the current state $\vu^s$ and then projects the state $\tilde\vu^{s+1}$ onto the submanifold $\M'$~\citep[Section IV.4]{Gear1986,Hairer2006}.
The projected state $\vu^{s+1}$ preserves the first integrals $V_k$.
In particular, the projected state $\vu^{s+1}$ is obtained by solving the optimization problem
\begin{equation}
  \arg\min_{\vu^{s+1}} \|\vu^{s+1}-\tilde\vu^{s+1}\| \mbox{ subject to } V_k(\vu^{s+1})-V_k(\vu^s)=0 \mbox{ for }k=1,\dots,K.\label{eq:projection_method}
\end{equation}
The local coordinate method defines a coordinate system on the neighborhood of the current state $\vu^s$ and integrates a differential equation on it~\citep[Section IV.5]{Potra1991,Hairer2006}.
The discrete gradient method defines a discrete analogue to a differential system and integrates it in discrete time, thereby preserving the Hamiltonian exactly (up to rounding errors) in discrete time~\citep{Furihata2010,Gonzalez1996,Hong2011}.

\paragraph{Neural Networks to Preserve First Integrals}
NODE defines the right-hand side $f$ of a differential system $\ddt\vu=f(\vu)$ using a neural network in the most general way with no associated first integrals~\citep{Chen2018e}.
NODE is a universal approximator to ODEs and can approximate any ODE with arbitrary accuracy if there is an infinite amount of training data~\citep{Teshima2020a}.
In practice, the amount of training data is limited, and prior knowledge about the target system is helpful for learning (see~\citet{Sannai2021} for the case with convolutional neural networks (CNNs)).
HNN~\citep{Greydanus2019} assumes the target system to be a Hamiltonian system in the canonical form, thereby guaranteeing various properties of Hamiltonian systems by definition, including the conservation of energy and preservation of the symplectic structure in continuous time~\citep{Hairer2006}.
Some studies have employed a symplectic integrator for HNN to preserve the energy and symplectic structure with smaller numerical errors~\citep{Chen2020a}.
LieConv and EMLP-HNN employ neural network architectures with translational and rotational symmetries to preserve momenta~\citep{Finzi2020,Finzi2021}.
CHNN incorporates a known holonomic constraint in the dynamics~\citep{Finzi2020b}.
Deep conservation extracts latent dynamics of a PDE system and preserves a quantity of interest by forcing its flux to be zero~\citep{Lee2019a}.
HNN++ also guarantees the conservation of mass in PDE systems by using a coefficient matrix derived from differential operators~\citep{Matsubara2020}.
These methods preserve known types of first integrals and suffer from temporal discretization errors.
In contrast, FINDE learns any types of first integrals from data and preserves them even after temporal discretization.

The neural projection method (NPM) learns fixed holonomic constraints using the projection (and inequality constraints)~\citep{Yang2020a}.
DGNet employed discrete gradient methods to guarantee the energy conservation in Hamiltonian systems (and the energy dissipation in friction systems)~\citep{Matsubara2020}.
While these methods preserve the aforementioned first integrals exactly in discrete time, their formulations are not available for other first integrals.

Several studies have proposed neural networks to learn Lyapunov functions, which are expected to be non-increasing over time, in contrast to first integrals~\citep{Manek2019,Takeishi2020}.
If the state moves in the direction of increasing the function, it is projected onto or moved inside the contour line of the Lyapunov function.
This concept is similar to that of the continuous-time version of FINDE but focuses on a single non-increasing quantity in continuous time; FINDE preserves multiple quantities in both continuous and discrete time.

\section{First Integral-Preserving Neural Differential Equation}\label{sec:method}
We suppose that a target system has at least $K$ unknown functionally independent first integrals.
When a neural network learns the dynamics of the target system, it is not guaranteed to learn these first integrals.
We suppose that a certain neural network $\hat f$ for modeling the target dynamics is given, and in addition to this model $\hat f$, we introduce a neural network that outputs a $K$-dimensional vector $\vV(\vu)=(V_1(\vu)\ V_2(\vu)\ \dots\ V_K(\vu))^\top$.
Each element is expected to learn one of the first integrals as $V_k:\R^N\rightarrow\R$ for $k=1,\dots,K$.
Then, the submanifold $\M'$ is defined as in Eq.~\eqref{eq:submanifold}.

\subsection{Continuous FINDE: Time-Derivative Projection Method}\label{sec:cFINDE}
We propose a time-derivative projection method called \emph{continuous FINDE (cFINDE)}.
The cFINDE projects the time-derivative onto the tangent space $\TuM'$.
Roughly speaking, the cFINDE projects the dynamics on the space of the directions in which the first integrals do not change.
In this way, the method can learn dynamics while preserving first integrals $V$, thereby finding unknown first integrals from data.

We refer to the neural network that defines the time-derivative $\hat f:\R^N\rightarrow\R^N$ as the base model.
Applying the method of Lagrange multipliers to the projection method in Eq.~\eqref{eq:projection_method}, and taking the limit as the time-step size approaches zero, we have
\begin{equation}
  \textstyle \ddt\vu=f(\vu), \
  f(\vu)=\hat f(\vu)-M(\vu)^\top\vlambda(\vu), \ \ddt \vV(\vu) =\mathbf{0}, \label{eq:cFINDE_base}
\end{equation}
where $M=\pderiv{\vV}{\vu}$ and $\vlambda\in\R^N$ is the Lagrange multiplier (see Appendix \ref{appendix:derivation} for detailed derivation).
We transform the second equation to obtain
\begin{equation}
  \mathbf{0}=\textstyle \ddt \vV(\vu(t))=\pderiv{\vV}{\vu} \ddt\vu=M(\vu) f(\vu)=M(\vu) (\hat f(\vu)-M(\vu)^\top\vlambda(\vu)),
  \label{eq:change_in_V}
\end{equation}
from which we obtain the Lagrange multiplier $\vlambda(\vu)=(M(\vu) M(\vu)^\top)^{-1}M(\vu)\hat f(\vu)$.
By eliminating $\vlambda(\vu)$, we define the cFINDE as
\begin{equation}
  \textstyle \ddt\vu=f(\vu)=(I-Y(\vu))\hat f(\vu) \text{ for } Y(\vu)=M(\vu)^\top(M(\vu) M(\vu)^\top)^{-1}M(\vu). \label{eq:cFINDE_equation}
\end{equation}
\begin{thm}[continuous-time first integral preservation]\label{thm:cFINDE}
  The cFINDE $\ddt\vu=f(\vu)$ preserves all first integrals $V_k$ for $k=1,\dots,K$ in continuous time, that is, $\ddt V_k=0$.
\end{thm}

See Appendix \ref{appendix:derivation} for proof.
The base model $\hat f$ can be a NODE, an HNN, or any other model depending on the available prior knowledge.
Additionally, if a first integral is already known, it can be directly used as one of the first integrals $V_k$ instead of being found by the neural network.
Note that even though the base model $\hat f$ is an HNN, due to the projection, the cFINDE $f$ is no longer a Hamiltonian system in the strict sense.

Compared to the base model $\hat f$, the cFINDE requires the additional computation of the neural network $\vV$, several matrix multiplications, and an inverse operation.
The inverse operation has a computational cost of $O(K^3)$, which is not costly if the number $K$ of first integrals is small.
Many previous models also need the inverse operation to satisfy the constraints and geometric structures, such as Lagrangian neural network (LNN)~\citep{Cranmer2020}, neural symplectic form~\citep{Chen2021NeurIPS}, and CHNN~\citep{Finzi2020b}.

\subsection{Discrete FINDE: Discrete-Time Derivative Projection Method}\label{sec:dFINDE}
The cFINDE is still an ODE and hence needs to be solved using a numerical integrator, which causes the temporal discretization errors in the first integrals.
In order to eliminate these errors, it is necessary to constrain the destination (i.e., finite difference) rather than the direction (i.e., time-derivative).
For this purpose, we propose \emph{discrete FINDE (dFINDE)} by employing \emph{discrete gradients} to define \emph{discrete tangent spaces}, which are needed to constraint the state variables on the submanifold $\M'$.

A discrete gradient $\bnabla V$ is a discrete analogue to a gradient $\nabla V$~\citep{Furihata2010,Gonzalez1996,Hong2011}.
Recall that a gradient $\nabla V$ of a function $V:\R^N\rightarrow\R$ can be regarded as a function $\R^N\rightarrow\R^N$ that satisfies the chain rule $\ddt V(\vu)=\nabla V(\vu)^\top\ddt\vu$.
Analogously, a discrete gradient $\bnabla$ is defined as follows:
\begin{dfn}[discrete gradient]\label{def:discrete_gradient}
  A discrete gradient $\bnabla V$ of a function $V:\R^N\rightarrow\R$ is a function $\R^N\times\R^N\rightarrow\R^N$ that satisfies
  $V(\vv)-V(\vu)=\bnabla V(\vv,\vu)^{\top}(\vv-\vu)$ and $\bnabla V(\vu,\vu)=\nabla V(\vu)$.
\end{dfn}
The first condition is a discrete analogue to the chain rule when replacing the time-derivatives $\ddt V$ and $\ddt\vu$ with finite differences $(V(\vv)-V(\vu))$ and $(\vv-\vu)$, respectively, and the second condition ensures consistency with the ordinary gradient $\nabla V$.
A discrete gradient $\bnabla V$ is not uniquely determined and has been obtained manually.
Recently, the automatic discrete differentiation algorithm (ADDA) has been proposed by \citet{Matsubara2020}, which obtains a discrete gradient of a neural network in a manner similar to the automatic differentiation algorithm~\citep{tensorflow,Paszke2017}.
The discrete gradient is defined in discrete time; hence, the prediction using the discrete gradient is free from temporal discretization errors.
See Appendix \ref{appendix:discrete_gradient} and the references~\citet{Furihata2010,Matsubara2020} for more details.

Following~\citet{Christiansen2011,Dahlby2011}, we introduce a discrete analogue to the tangent space $\TuM'$ called the discrete tangent space $\TvuM'$.
In particular, for a pair of points $(\vv,\vu)\in\M'$, the discrete tangent space is defined as
\begin{equation}
  \TvuM'=\{\vw\in\R^N: \bnabla V_k(\vv,\vu)^{\top}\vw=0 \mbox{ for } k=1,\dots,K\}.
\end{equation}
If the finite difference $(\vu^{s+1}-\vu^{s})$ between the predicted and current states is on the discrete tangent space $\TuuM'$, the first integrals $V_k$ are preserved because $V_k(\vu^{s+1})-V_k(\vu^{s})=\bnabla V_k(\vu^{s+1},\vu^{s})^\top(\vu^{s+1}-\vu^{s})=0$.
Note that similar concepts defined in different ways are also referred to as discrete tangent spaces~\citep{Cuell2009,Dehmamy2021}.

We suppose that a neural network (e.g., NODE) $\hat f$ defines an ODE and a numerical integrator predicts the next state $\tilde\vu^{s+1}$ from a given state $\vu^{s}$.
We call this process a discrete-time base model $\hat\psi$, which satisfies $\frac{\tilde \vu^{s+1}-\vu^{s}}{\Delta t^s}=\hat\psi(\vu^{s};\Delta t^s)$.
Subsequently, we consider the model
\begin{equation}
  \begin{aligned}
     & \textstyle \frac{\vu^{s+1}-\vu^{s}}{\Delta t^s}=\psi(\vu^{s+1},\vu^{s};\Delta t^s),                                                                                \\
     & \psi(\vu^{s+1},\vu^{s};\Delta t^s)=\hat\psi(\vu^{s};\Delta t^s)-\bM(\vu^{s+1},\vu^{s})^{\top}\vlambda(\vu^{s+1},\vu^{s}),\ \vV(\vu^{s+1})-\vV(\vu^{s})=\mathbf{0},
  \end{aligned}\label{eq:psi}
\end{equation}
where $\bM(\vu^{s+1},\vu^{s})=(\bnabla V_1(\vu^{s+1},\vu^{s})\ \dots\ \bnabla V_K(\vu^{s+1},\vu^{s}))^\top$.
As shown in Appendix \ref{appendix:derivation}, this formulation is also derived from the projection method in Eq.~\eqref{eq:projection_method}.
Using the chain rule of the discrete gradient,
\begin{equation}
  \textstyle \mathbf{0}
  =\frac{\vV(\vu^{s+1})-\vV(\vu^{s})}{\Delta t^s}
  =\bM(\vu^{s+1},\vu^{s})\frac{\vu^{s+1}-\vu^{s}}{\Delta t^s}
  =\bM(\vu^{s+1},\vu^{s})\psi(\vu^{s+1},\vu^{s};\Delta t^s),
\end{equation}
Substituting this into Eq.~\eqref{eq:psi} and eliminating the Lagrange multiplier $\vlambda$, we define the dFINDE as
\begin{equation}
  \textstyle \textstyle \frac{\vu^{s+1}-\vu^{s}}{\Delta t^s}=\psi(\vu^{s+1},\vu^{s};\Delta t^s)=(I-\bY(\vu^{s+1}\!\!,\vu^{s}))\hat\psi(\vu^{s};\Delta t^s)
  \text{ for } \bY=\bM^\top(\bM\ \bM^\top)^{-1}\bM,
  \label{eq:update_dFINDE}
\end{equation}
where we have abbreviated $\bM(\vu^{s+1},\vu^{s})$ and $\bY(\vu^{s+1},\vu^{s})$ to $\bM$ and $\bY$, respectively.

\begin{thm}[discrete-time first integral preservation]\label{thm:dFINDE}
  The dFINDE $\frac{\vu^{s+1}-\vu^{s}}{\Delta t^s}=\psi(\vu^{s+1}\!\!,\vu^{s};\Delta t^s)$ preserves all first integrals $V_k$ for $k=1,\dots,K$ in discrete time, that is, $V_k(\vu^{s+1})-V_k(\vu^{s})=0$.
\end{thm}

See Appendix \ref{appendix:derivation} for proof.
Intuitively, dFINDE projects the finite difference (discrete-time derivative) $\hat\psi$ onto the discrete tangent space $\TuuM'$ after the numerical integration for each step, whereas cFINDE projects the time-derivative $\hat f$ onto the tangent space $\TuM'$ at every substep inside a numerical integrator.
In the discrete-time base model $\hat\psi$, the ODE $\hat f$ can be defined by any model, such as NODE or HNN, and the numerical integrator can be implemented by any method, such as the Runge--Kutta method or the leapfrog integrator.
The projection method in Eq.~\eqref{eq:projection_method}, the method in Eq.~\eqref{eq:psi}, and the dFINDE in Eq~\eqref{eq:update_dFINDE} are implicit methods and hence relatively computationally expensive.
However, only the dFINDE can be trained non-iteratively by standard backpropagation algorithms.
As explained in Appendix \ref{appendix:training_objective}, this is because the next state $\vu^{s+1}$ is given during training and the ADDA can explicitly obtain the discrete gradient and its computational graph.

\begin{table}[t]
  \fontsize{8pt}{9pt}\selectfont\centering
  \caption{Datasets, Dynamics, and First Integrals.}
  \label{tab:datasets}
  \begin{tabular}{llccccc}
    \toprule
                               &                               &              & \multicolumn{4}{c}{\textbf{First Integrals}}                                                                                   \\
    \cmidrule(lr){4-7}
    \textbf{Dataset}           & \textbf{Dynamics (Structure)} & \textbf{$N$} & \!\!\textbf{Energy}\!\!                      & \!\!\textbf{Momentum}\!\! & \!\!\textbf{Mass}\!\! & \!\!\textbf{Constraint}\!\! \\
    \midrule
    Two-body problem           & Canonical Hamiltonian         & 8            & \checkmark                                   & \checkmark                                                                      \\
    Discretized KdV equation   & Non-canonical Hamiltonian     & 50           & \checkmark                                   &                           & \checkmark                                          \\
    Double pendulum            & Poisson                       & 8            & \checkmark                                   &                           &                       & \checkmark                  \\
    FitzHugh--Nagumo model\!\! & Dirac                         & 4            &                                              &                           &                       & \checkmark                  \\
    \bottomrule
  \end{tabular}
  \vspace*{-3mm}
\end{table}

\section{Experiments}
\subsection{Experimental Settings}\label{sec:setting}
\paragraph{Target Systems}
We evaluated FINDE and base models using datasets associated with first integrals; these are summarized in Table~\ref{tab:datasets}.
A gravitational two-body problem (2-body) on a 2-dimensional configuration space is a typical Hamiltonian system in the canonical form.
In addition to the total energy, the system has first integrals related to symmetries in space, namely, the linear and angular momenta.
The Korteweg--De Vries (KdV) equation is a PDE model of shallow water waves.
This equation is a Hamiltonian system in a non-canonical form and has the Hamiltonian, total mass, and many other quantities as first integrals.
We discretized the KdV equation in space, obtaining a fifty-dimensional state $\vu$.
A double pendulum (2-pend) is a Hamiltonian system in polar coordinates.
However, we transformed it to Cartesian coordinates; hence, it became a Poisson system.
The lengths of the two rods work as holonomic constraints and lead to four first integrals in addition to the Hamiltonian.
The FitzHugh--Nagumo model is a biological neuron model as an electric circuit, which exhibits a rapid and transient change of voltage called a spike.
As an electric circuit, the currents through and voltages applied to the inductor and capacitor can be regarded as system states, which are constrained by the circuit topology and Kirchhoff's current and voltage laws.
Then, this system has a state of four elements and two first integrals.
Because the resistor dissipates the energy, the system is not a Poisson system, but a Dirac structure can be found~\citep{VanderSchaft2014}.
We generated a time-series set of each dataset with different initial conditions (hence, different values of first integrals).
See Appendix \ref{appendix:datasets} for more details.

\paragraph{Implementation}
We implemented the proposed FINDE and evaluated it under the following settings.
We implemented all codes by modifying the officially released codes of HNN~\citep{Greydanus2019} \footnote{\url{https://github.com/greydanus/hamiltonian-nn}} and DGNet~\citep{Matsubara2020}\footnote{\url{https://github.com/tksmatsubara/discrete-autograd}\label{footnote:dgnet}}.
We used Python v.~3.8.12 with packages scipy v.~1.7.3, pytorch v.~1.10.2, torchdiffeq v.~0.1.1, functorch v.~1.10 preview, and gplearn v.~0.4.2.
We used the Dormand--Prince method (dopri5)~\citep{Dormand1986} as the numerical integrator, except in Section~\ref{sec:first_integral_preservation}.
All experiments were performed on a single NVIDIA A100.

Following HNN~\citep{Greydanus2019} and DGNet~\citep{Matsubara2020}, we used fully-connected neural networks with two hidden layers.
The input was the state $\vu$, and the output represented the first integrals $\vV$ for FINDE, time-derivative $\hat f$ for NODE, or the Hamiltonian $H$ for HNN.
Each hidden layer had 200 units and preceded a hyperbolic tangent activation function.
Each weight matrix was initialized as an orthogonal matrix.
For the KdV dataset, we used a 1-dimensional CNN, wherein the kernel size of each layer was 3.
The double pendulum is a second--order system, implying that the time-derivative $\ddt\vq$ of the position $\vq$ is known to be the velocity $\vv$.
Hence, we treated only the acceleration $\ddt\vv$ as the output to learn in the 2-pend dataset.
This assumption slightly improved the absolute performances but did not change the relative trends.

As the loss function for the cFINDE, we used the mean squared error (MSE) between the ground truth future state $\vu^{s+1}_\textrm{GT}$ and the future state $\vu^{s+1}_\textrm{pred.}$ predicted from the current step $\vu^{s}_\textrm{GT}$ normalized by the time-step size $\Delta t^s$; we named this the \emph{1-step error}.
For the dFINDE, we used the MSE between the left- and right-hand sides of Eq.~\eqref{eq:update_dFINDE} because the ground truth states $\vu^{s}_\textrm{GT}$ and $\vu^{s+1}_\textrm{GT}$ are available during the training phase.
The base model and FINDE were jointly trained using the Adam optimizer~\citep{Kingma2014b} with the parameters $(\beta_1,\beta_2)=(0.9,0.999)$ and a batch size of 200.
The learning rate was initialized to $10^{-3}$ and decayed to zero with cosine annealing~\citep{Loshchilov2017}.
See Appendix \ref{appendix:training_objective} and the enclosed source code for details about implementations.

\paragraph{Evaluation Metric}
We used the 1-step error as an evaluation metric, which is identical to the loss function for the cFINDE, and displayed it in the scale $\times 10^{-9}$.
The lower this indicator, the better, as indicated by $\downarrow$.
The MSEs of the state or system energy over a long period are misleading indicators, as suggested in prior studies~\citep{Botev2021,Jin2020,Vlachas2020}.
For example, a periodic orbit that is correctly learned except for a slight difference in angular velocity would have the same MSE as an orbit that never moves from its initial position.
Instead, we used the valid prediction time (\emph{VPT})~\citep{Botev2021,Jin2020,Vlachas2020}.
VPT denotes the time point $s$ divided by the length $S$ of the time-series at which the MSE of the predicted state $\vu^{s}_\textrm{pred.}$ first exceeds a given threshold $\theta$ in an initial value problem, that is,
\begin{equation}
  \textstyle VPT(\vu_\textrm{pred.};\vu_\textrm{GT})=\frac{1}{S} \max \{s_f | \mathrm{MSE}(\vu^s_\textrm{pred.},\vu^s_\textrm{GT})<\theta\mbox{ for all }s\le s_f,\ 0\le s_f\le S\}.
\end{equation}
The higher this indicator, the better, as indicated by $\uparrow$.
To obtain VPTs, we normalized each element of state to have the zero mean and unit variance in the training data and set $\theta$ to 0.01.
For systems with ``spiking'' behaviors, a small error in phase may be regarded as a significant error in the state; for the FitzHugh--Nagumo model, we obtained the VPTs by allowing for a delay and advance of up to 5 steps.

\begin{wrapfigure}{r}{1.7in}
  \vspace*{-18mm}
  \includegraphics[scale=0.9]{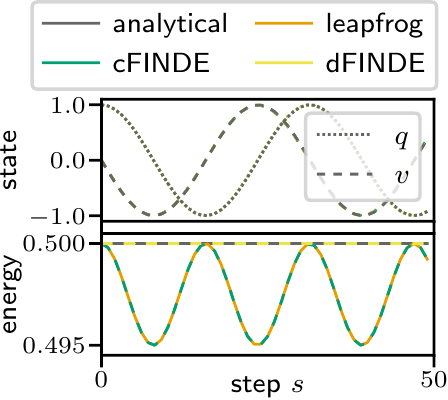}
  \vspace*{-2mm}
  \caption{Integration of a known mass-spring system by the leapfrog integrator.
    (top) States predicted by comparison methods.
    (bottom) Energy calculated from the states predicted.
  }\label{fig:mass_spring}
  \vspace*{-8mm}
\end{wrapfigure}

\subsection{Demonstration of First Integral Preservation}\label{sec:first_integral_preservation}
Before learning first integrals from data, we demonstrate that dFINDE can preserve first integrals without temporal discretization errors.
We used a mass-spring system, which had the state $\vu=(q\ v)^\top\!$, dynamics $\ddt q=v$ and $\ddt v=-q$, and system energy $E(q,v)=\frac{1}{2}(q^2+v^2)$.
Using an initial value of $(1.0\ \ 0.0)^\top$ and a time-step size of $\Delta t=0.2$, we solved the initial value problem of the true ODE using the leapfrog integrator with or without FINDE, with the true system energy $E$ as the first integral $V$.
Notably, no neural networks nor training were involved.

Figure \ref{fig:mass_spring} shows the results, along with the analytical solution.
The states predicted by comparison methods overlap and are apparently identical.
However, the energy obtained by the leapfrog integrator fluctuates and the same is true for cFINDE.
This is because the leapfrog integrator and cFINDE suffer from temporal discretization errors in first integrals.
In contrast, dFINDE preserves the energy accurately, the same as the analytical solution.
This is because dFINDE projects the state $(q\ v)^\top$ onto the discrete tangent space $\TvuM'$ at every step.
Although a smaller time-step size reduces temporal discretization errors, this result demonstrates the advantage of dFINDE.
See Appendix~\ref{appendix:mass_spring_dopri} for the case with the Dormand-Prince integrator.

\subsection{Finding Non-Hamiltonian First Integrals of Hamiltonian Systems}\label{sec:FINDE_hnn}
We evaluated cFINDE and dFINDE on learning from the 2-body dataset.
We used HNN as the base model $\hat f$.
We found that cFINDE and dFINDE obtained better performances if it did not treat the Hamiltonian $H$ of the HNN as one of the first integrals $V_k$.
The medians and standard deviations of five trials are summarized in the leftmost column of Table~\ref{tab:performance}.
The cFINDE achieved better VPTs than the original HNN with $K=1$ to $2$, and its performance was suddenly degraded with $K=3$.
The dFINDE showed a similar trend with slightly better performances; there is a trade-off between performance and computational cost.
The HNN with either cFINDE or dFINDE found two first integrals in addition to the Hamiltonian $H$ of the HNN.
Even though a two-body problem is a Hamiltonian system that an HNN can learn, the prior knowledge that there exist first integrals other than the Hamiltonian $H$ can be a clue that enables better learning.
Despite their better long-term prediction performance, the HNN with either cFINDE or dFINDE yielded 1-step errors worse than the HNN, indicating that the 1-step error is misleading as an evaluation criterion.

These example results are depicted in Fig.~\ref{fig:2body}.
In the absence of FINDE, the mass positions $(x_1,y_1)$ and $(x_2,y_2)$ became inaccurate in a short time and the center-of-gravity position $(x_c,y_c)=(\frac{x_1+x_2}{2},\frac{y_1+y_2}{2})$ deviated rapidly.
The HNN with cFINDE accurately predicted the state for a longer period.
Even after errors in the mass positions became non-negligible, errors in the center-of-gravity position were still small.
Figure~\ref{fig:2body_error} shows the absolute errors averaged over all trials, which demonstrate how the trend changes with cFINDE.
In both the $x$- and $y$-directions, the HNN without FINDE produced errors in the center-of-gravity position $x_c$ (or $y_c$), and those in the mass positions $x_1,x_2$ (or $y_1,y_2$) at a similar level.
In contrast, with the cFINDE, errors in the center-of-gravity position were much smaller than those in the mass positions, implying that errors in one mass position canceled out errors in the other.
We performed a symbolic regression of first integrals $\vV$ found by the neural network.
For $K=2$, the found first integrals $\vV$ were identical to the linear momenta in the $x$- and $y$-directions up to affine transformation in most cases.
See Appendix \ref{appendix:symbolic_regression} for detailed results.
Therefore, we conclude that FINDE not only had better prediction accuracy but also found and preserved linear momenta (which are related to symmetries in space) more accurately despite not having prior knowledge about symmetries.

\begin{table}[t]
  \centering
  \fontsize{8pt}{9pt}\selectfont
  \tabcolsep=.25mm
  \captionof{table}{Results of cFINDE and dFINDE.}
  \label{tab:performance}
  \begin{tabular}{lcrlrlrlrl}
    \toprule
               &     & \multicolumn{2}{c}{\textbf{2-body + HNN}}       & \multicolumn{2}{c}{\textbf{KdV}}           & \multicolumn{2}{c}{\textbf{2-pend}}             & \multicolumn{2}{c}{\textbf{FitzHugh--Nagumo}}
    \\
    \cmidrule(lr){3-4}\cmidrule(lr){5-6}\cmidrule(lr){7-8}\cmidrule(lr){9-10}
    Model      & $K$ & \multicolumn{1}{c}{\textbf{1-step}$\downarrow$} & \multicolumn{1}{c}{\textbf{VPT}$\uparrow$} & \multicolumn{1}{c}{\textbf{1-step}$\downarrow$} & \multicolumn{1}{c}{\textbf{VPT}$\uparrow$}    & \multicolumn{1}{c}{\textbf{1-step}$\downarrow$} & \multicolumn{1}{c}{\textbf{VPT}$\uparrow$} & \multicolumn{1}{c}{\textbf{1-step}$\downarrow$} & \multicolumn{1}{c}{\textbf{VPT}$\uparrow$} \\
    \midrule
    base model & --  & 5.17\std{0.57\zz}                               & 0.362\std{0.026}                           & 5.59\std{0.30\zz}                               & 0.339\std{0.038}                              & 0.82\std{0.02\zz}                               & 0.110\std{0.035}                           & 73.66\std{12.59}                                & 0.236\std{0.053}                           \\
    \midrule
               & 1   & 7.10\std{1.25\zz}                               & \ul{0.374}\std{0.036}                      & 6.24\std{0.44\zz}                               & \ul{0.371}\std{0.088}                         & \ul{0.75}\std{0.04\zz}                          & \ul{0.156}\std{0.042}                      & \ul{54.18}\std{8.12\zz}                         & 0.127\std{0.148}                           \\
               & 2   & 7.78\std{1.39\zz}                               & \ub{0.450}\std{0.052}                      & \ub{2.59}\std{0.11\zz}                          & \ul{0.608}\std{0.085}                         & \ul{0.73}\std{0.05\zz}                          & \ul{0.198}\std{0.088}                      & \ub{37.03}\std{3.81\zz}                         & \ub{0.437}\std{0.084}                      \\
    + cFINDE   & 3   & \multicolumn{1}{c}{$>\!10^3$}                   & 0.147\std{0.146}$^*\!\!$                   & \ul{3.19}\std{0.37\zz}                          & \ub{0.730}\std{0.091}                         & \ub{0.69}\std{0.03\zz}                          & \ul{0.411}\std{0.093}                      & \multicolumn{1}{c}{$>\!10^6$}                   & 0.007\std{0.007}$^*\!\!$                   \\
               & 4   & \multicolumn{1}{c}{$>\!10^3$}                   & 0.101\std{0.005}                           & \ul{3.65}\std{0.30\zz}                          & \ul{0.641}\std{0.071}                         & \ul{0.77}\std{0.07\zz}                          & \ul{0.395}\std{0.083}                      & \multicolumn{1}{c}{---}                                                                      \\
               & 5   & \multicolumn{1}{c}{$>\!10^3$}                   & 0.080\std{0.014}                           & \ul{4.68}\std{0.43\zz}                          & \ul{0.601}\std{0.069}                         & \ul{0.80}\std{0.07\zz}                          & \ub{0.585}\std{0.097}                      & \multicolumn{1}{c}{---}                                                                      \\
               & 6   & \multicolumn{1}{c}{$>\!10^3$}                   & 0.070\std{0.019}                           & 7.79\std{0.51\zz}                               & \ul{0.425}\std{0.067}                         & 12.53\std{0.00\zz}                              & 0.005\std{0.000}$^*\!\!$                   & \multicolumn{1}{c}{---}                                                                      \\
    \midrule
               & 1   & 7.01\std{1.06\zz}                               & \ul{0.379}\std{0.040}                      & 11.61\std{6.60\zz}                              & 0.288\std{0.083}                              & \ul{0.75}\std{0.10\zz}                          & \ul{0.152}\std{0.017}                      & \ul{47.07}\std{8.03\zz}                         & 0.117\std{0.122}                           \\
               & 2   & 7.03\std{1.00\zz}                               & \ub{0.475}\std{0.022}                      & \ub{2.70}\std{0.26\zz}                          & \ul{0.598}\std{0.059}                         & \ul{0.74}\std{0.05\zz}                          & \ul{0.271}\std{0.111}                      & \ub{33.24}\std{3.40\zz}                         & \ub{0.455}\std{0.032}                      \\
    + dFINDE   & 3   & 54.78\std{36.39}                                & 0.309\std{0.024}                           & \ul{3.78}\std{0.27\zz}                          & \ul{0.636}\std{0.024}                         & \ub{0.69}\std{0.05\zz}                          & \ul{0.447}\std{0.081}                      & 319.70\std{91.11}                               & 0.049\std{0.007}                           \\
               & 4   & \multicolumn{1}{c}{$>\!10^3$}                   & 0.102\std{0.015}                           & \ul{3.48}\std{0.32\zz}                          & \ub{0.780}\std{0.059}                         & \ul{0.71}\std{0.03\zz}                          & \ul{0.454}\std{0.060}                      & \multicolumn{1}{c}{---}                                                                      \\
               & 5   & \multicolumn{1}{c}{$>\!10^3$}                   & 0.086\std{0.011}$^*\!\!$                   & \ul{5.26}\std{0.15\zz}                          & \ul{0.718}\std{0.038}                         & 0.86\std{0.09\zz}                               & \ub{0.591}\std{0.087}                      & \multicolumn{1}{c}{---}                                                                      \\
               & 6   & \multicolumn{1}{c}{$>\!10^3$}                   & 0.059\std{0.017}                           & 9.60\std{3.61\zz}                               & \ul{0.573}\std{0.121}                         & 58.88\std{22.98}                                & 0.037\std{0.039}                           & \multicolumn{1}{c}{---}                                                                      \\
    \bottomrule
  \end{tabular}\\[1pt]
  \raggedright
  \footnotesize
  Notes: Standard deviation follows the $\pm$ symbol; underlined results are better than those of the base models; bold font indicates best results; $^*$ denotes trials that failed in training because of underflow of time-step size.
  \vspace*{-2mm}
\end{table}

\begin{figure}[t]
  \begin{minipage}{2.8in}
    \raggedright
    \includegraphics[scale=0.9]{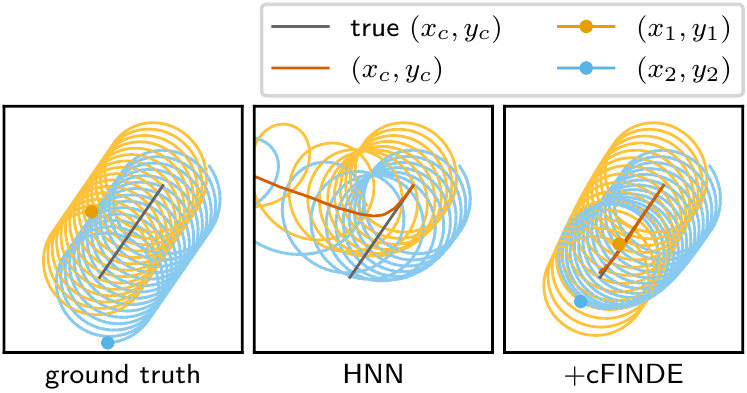}
    \begin{minipage}{2.7in}
      \caption{Example results of 2-body dataset.
        (left) Ground truth.
        (middle) HNN.
        (right) HNN with cFINDE.
      }
      \label{fig:2body}
    \end{minipage}
  \end{minipage}
  \begin{minipage}{2.6in}
    \raggedleft
    \vspace*{-3mm}
    \centering\includegraphics[scale=0.9]{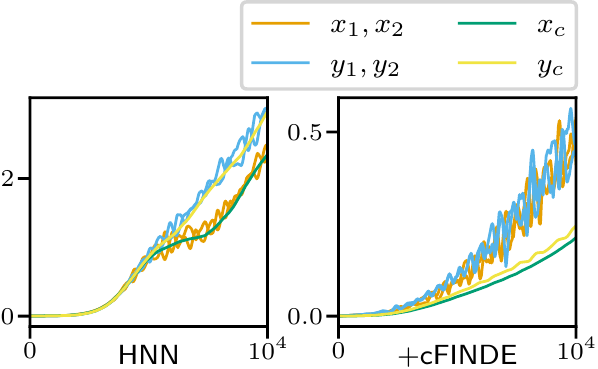}
    \vspace*{-2mm}
    \caption{
      Mean absolute errors for results of 2-body dataset.
      (left) HNN.
      (right) HNN with cFINDE.
    }
    \label{fig:2body_error}
  \end{minipage}
  \vspace*{-8mm}
\end{figure}

\subsection{Finding First Integrals of Unknown Systems}\label{sec:FINDE_node}
It is often unclear whether a target system is a Hamiltonian system or not, but one can expect that it has several first integrals.
We evaluated cFINDE and dFINDE using NODE as the base model and display the results in Table~\ref{tab:performance}.

For the KdV dataset, the NODE with either cFINDE or dFINDE obtained improved VPTs for a wide range of $K$.
Figure~\ref{fig:kdv} shows an example result.
The prediction states were apparently similar.
In the absence of FINDE, the NODE increased all of its errors in proportion to time.
With cFINDE, the error in total mass increased at the point where the two solitons collided, but then returned to the original level.
Although the calculation is slightly inaccurate, the cFINDE learned to preserve the total mass.
The error in energy continued to increase for $K=2$, but remained within a small range for $K=3$.
These results suggest that the first or second quantity learned by the cFINDE was total mass, the third quantity was system energy, and the remaining quantity may correspond to one of the many first integrals of the KdV equation.

For the 2-pend dataset, the NODE with either cFINDE or dFINDE obtained improved VPTs with $K=1$ to $5$.
In addition to the system energy, the double pendulum has two holonomic constraints on the position, which lead to two additional constraints involving the velocity (see Appendix \ref{appendix:datasets} for details).
Thus, it is reasonable that the NODE with either cFINDE or dFINDE obtained the best VPT for $K=5$ first integrals and completely failed for $K>5$ first integrals.
As exemplified in Fig.~\ref{fig:2pend_result}, the NODE without FINDE did not preserve the lengths of rods, making the states deviate gradually.
See Appendix \ref{appendix:holonomic_constraints} for the case when actual constraints are known.
For the FitzHugh--Nagumo dataset, the NODE with either cFINDE or dFINDE obtained improved VPTs for $K=2$.
As exemplified in Fig.~\ref{fig:fn_result}, the ground truth state converged to a periodic orbit, and only the NODE with cFINDE for $K=2$ reproduced similar dynamics.
Without FINDE, the state did not remain in a limited region.
For $K=1$, the state converged to a wrong equilibrium; the sole quantity $V_1$ may have attempted and failed to learn both first integrals.
We conclude that both cFINDE and dFINDE found all first integrals of the 2-pend and FitzHugh--Nagumo datasets; $K=5$ and $K=2$, respectively.

\begin{figure}[t]
  \centering
  \includegraphics[scale=0.9]{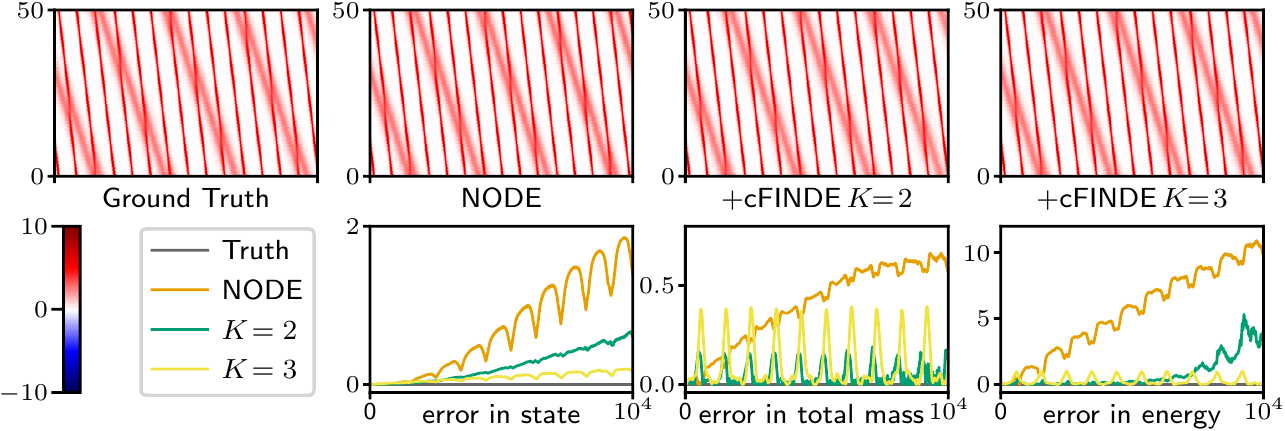}
  \vspace*{-2mm}
  \caption{Example results of KdV dataset.
    (top) Predicted states.
    Red belts denote moving solitons.
    (bottom) Mean absolute errors in states $\vu$, total mass $\sum_{k=1}^N u_k$, and energy, from left to right.}
  \label{fig:kdv}
  \vspace*{1mm}
  \begin{minipage}{2.3in}
    \centering
    \includegraphics[scale=0.9]{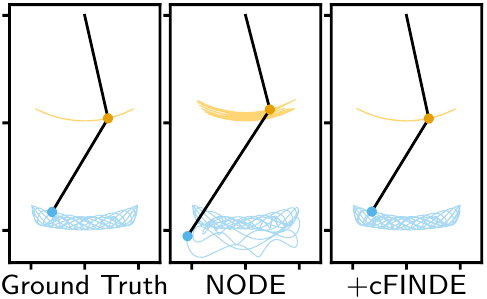}
    \vspace*{-2mm}
    \caption{Example results of 2-pend dataset for 2,000 steps.
      (left) Ground truth.
      (middle) NODE.
      (right) NODE with cFINDE for $K=5$.
    }
    \label{fig:2pend_result}
  \end{minipage}\hfill
  \begin{minipage}{3.0in}
    \raggedleft
    \includegraphics[scale=0.9]{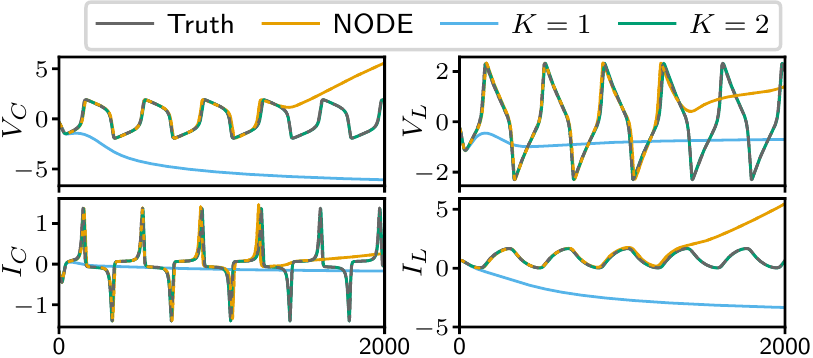}
    \vspace*{-2mm}
    \caption{Example results of FitzHugh--Nagumo dataset.
      Each panel shows one of four states.
    }
    \label{fig:fn_result}
  \end{minipage}
  \vspace*{-4mm}
\end{figure}

\section{Conclusion}
This study proposed \emph{first integral-preserving neural differential equation} (FINDE), which can find and preserve any type of first integrals from data in a unified manner.
FINDE projects the time evolution onto the submanifold defined using the (discrete) gradients of first integrals represented by a neural network.
We experimentally demonstrated that FINDE found and preserved first integrals that come from the energy and mass conservation laws, symmetries in space, and constraints, thereby predicting the dynamics for far longer.
FINDE is available even for an energy-dissipating system.
When FINDE obtains the best prediction accuracy with $K=K'$, it suggests that the target system has at least $K'$ first integrals.
Hence, FINDE has the potential to make scientific discoveries by revealing geometric structures of dynamical systems.
See Appendix~\ref{appendix:how_to_k} for more discussions on $K$.

The numerical error tolerance $10^{-9}$ was negligible compared to the 1-step errors (which were $10^{-5}$ to $10^{-4}$ in absolute error).
However, the dFINDE tended to obtain much better VPTs than the cFINDE.
This result suggests that a method leading to smaller numerical errors produces a model with smaller modeling errors, as observed in previous works~\citep{Chen2020a,Matsubara2020}.
These results may form a new frontier for integrating numerical and modeling errors.

\section*{Reproducibility Statement}
See Section~\ref{sec:setting} for experimental settings.
More detailed descriptions can be found in Appendix~\ref{appendix:training_objective} for training procedure and Appendix~\ref{appendix:datasets} for datasets.
The authors have enclosed the source code for generating the datasets and running the experiments as supplementary material.

\section*{Acknowledgement}
This study was partially supported by JST CREST (JPMJCR1914), JST PRESTO (JPMJPR21C7), and JSPS KAKENHI (19K20344, 20K11693).


\clearpage
\newpage
\appendix
\renewcommand\thetable{A\arabic{table}}
\setcounter{table}{0}
\renewcommand\thefigure{A\arabic{figure}}
\setcounter{figure}{0}
\renewcommand\theequation{A\arabic{equation}}
\setcounter{equation}{0}

\section{Hamiltonian System, its Generalization, and First Integrals}\label{appendix:system}
\paragraph{Preliminary}
In this section, we briefly introduce potential target systems and related works.
Methods proposed by related works use specific prior knowledge about target systems, such as constraints.
In contrast, our proposed FINDE assumes a situation where neural networks learn systems with unknown properties.
See, for example, \citet{Hairer2006,VanderSchaft2014} for more details about geometric mechanics.

On an $N$-dimensional manifold $\M$, an ODE is defined using a vector field $f:\M\rightarrow\TuM$, which maps a point $\vu$ on the manifold $\M$ to a tangent vector $f(\vu)$ on the tangent space $\TuM$.
The NODE defines an ODE in this way~\citep{Chen2018e}.
Given a scalar-valued function $H:\M\rightarrow\R$ on the manifold $\M$, its differential $\dif H:\M\rightarrow\TcuM$ is a cotangent vector field (a.k.a.~a differential 1-form), which maps a point $\vu$ on the manifold $\M$ to a cotangent vector $\dif H(\vu)$ on the cotangent space $\TcuM$.

\paragraph{Hamiltoanian System}
A Hamiltonian system is defined using a non-degenerate closed differential 2-form $\omega$ called symplectic form, which is a skew-symmetric bilinear map $\omega_\vu:\TuM\times\TuM\rightarrow\R$ at point $\vu$.
A symplectic form assigned to a manifold is called the symplectic structure.
The coordinate-free form of Hamilton's equation is $\ddt\vu=X_H(\vu),\ \omega_\vu(X_H(\vu),\vw)=\langle\dif H(\vu),\vw\rangle$ for any $\vw\in\TuM$, where $X_H$ is the Hamiltonian vector field.
The symplectic form $\omega$ gives rise to a bundle map $\omega^\flat_\vu:\TuM\rightarrow\TcuM$, with which Hamilton's equation is rewritten as $\ddt \vu=X_H(\vu)=(\omega^\flat_\vu)^{-1}(\dif H(\vu))$.
The right-hand side is locally equivalent to the product of a coefficient matrix $S$ and the gradient $\nabla H$ of the Hamiltonian $H$.
Then, Hamilton's equation is obtained as $\ddt \vu=S\nabla H(\vu)$.
Hamiltonian systems are often expressed in the canonical form, in other words, they are defined on Darboux coordinates, on which the state $\vu$ is the paired generalized position $\vq$ and generalized momentum $\vp$.
The corresponding coefficient matrix is $S=\begin{psmallmatrix} 0 & I_n \\ -I_n & 0\end{psmallmatrix}$ for $2n=N$ and the $n$-dimensional identity matrix $I_n$.
The HNN was developed to model Hamiltonian systems in the canonical forms~\citep{Greydanus2019}.

An Euler--Lagrange equation with a hyperregular Lagrangian and a Lotka--Volterra equation are also Hamiltonian systems; however, their coordinate systems are not Darboux coordinates.
A neural symplectic form (NSF) handles this class of equations~\citep{Chen2021NeurIPS}.
The KdV equation is also a Hamiltonian system not on Darboux coordinates.
For Hamiltonian PDE systems, HNN++ was proposed~\citep{Matsubara2020}.
According to Darboux's theorem, any Hamiltonian system on an even--dimensional manifold can be transformed into the canonical form.

Noether's theorem states that a continuous symmetry of a system leads to a conservation law.
A Hamiltonian system is symmetric (invariant) to translation in time and conserves the Hamiltonian $H$.
A two-body problem is symmetric to translation and rotation in space and conserves linear and angular momenta.
These quantities are first integrals.
LieConv and EMLP-HNN had such symmetries implemented in their architectures~\citep{Finzi2020,Finzi2021}.
A pendulum is not symmetric to translation and rotation in space and does not conserve linear and angular momenta, but does exchange them with the base to which it is fixed.

\paragraph{Poisson System}
A Poisson system is named after a Poisson bracket $\{\cdot,\cdot\}$, but it is convenient to refer to it as a degenerate Hamiltonian system.
A Poisson bracket is defined using a Poisson 2-vector $B$, which is a skew-symmetric bilinear map $B_\vu:\TcuM\times\TcuM\rightarrow\R$ at point $\vu$.
The Poisson 2-vector $B$ gives rise to a bundle map $B^\sharp_\vu:\TcuM\rightarrow\TuM$ and defines Hamilton's equation as $\ddt \vu=B^\sharp(\dif H(\vu))$.
The Darboux--Lie theorem states that any Poisson system can be transformed into the canonical form $\ddt \vu=S\nabla H(\vu)$ by using a matrix $S=\begin{psmallmatrix} 0 & I_k & 0 \\ -I_k & 0 & 0 \\ 0&0&0\end{psmallmatrix}$ for $2k<N$.
The last $N-2k$ elements remain unchanged and correspond to the first integrals.
In this sense, a Poisson system is a degenerate Hamiltonian system.
A Poisson 2-vector assigned to a manifold is called a Poisson structure.
Several models of the dynamics of disease spreading and chemical reactions are Poisson systems, and total population and molecular mass are typical first integrals.

A Poisson neural network (PNN) learns to transform a given Poisson system into a canonical form~\citep{Jin2020a}.

\paragraph{Constrained Hamiltonian System}
A constraint $C(\vq)=0$ on the position $\vq$ is called a holonomic constraint.
Holonomic constraint appear, for example, when the arm's length restricts the position of a robot's hand.
Differentiating a holonomic constraint $C(\vq)=0$ yields a constraint involving the velocity $G(\vq,\vv)=\pderiv{C}{\vq}\vv=0$, which is simply called a velocity constraint.
Hence, each holonomic constraint leads to two first integrals $C$ and $G$.
A Hamiltonian system with holonomic constraints is also a Poisson system; in particular, it is a constrained Hamiltonian system.

A CHNN incorporates the known holonomic constraints $C(\vq)$ and corresponding velocity constraints $G(\vq,\vv)$ of a Hamiltonian system in the canonical form~\citep{Finzi2020b}.
The original study suggested that CHNN may learn holonomic constraints from data, but this has not been tested.
For modeling a constrained Hamiltonian system, it is sufficient to incorporate only velocity constraints $G(\vq,\vv)$ because a holonomic constraint $C(\vq)$ is implicitly satisfied if the corresponding velocity constraint $G(\vq,\vv)$ is satisfied.
\citet{Celledoni2022} used such formulation, and extended HNN and CHNN to systems on non-Euclidean spaces.
A neural projection method learns fixed holonomic constraints, as well as inequality constraints, which are outside the scope of this study~\citep{Yang2020a}.
This method updates the state by solving an optimization problem similar to Eq.~\eqref{eq:projection_method} iteratively using the gradient descent method at every training step.
Subsequently, it applies the backpropagation algorithm to all the optimization iterations.
Thus, it has high computational and memory costs.

These studies mainly focused on physically-induced holonomic constraints and may not work for other first integrals, as shown in Appendices \ref{appendix:holonomic_constraints} and \ref{appendix:neural_projection}.
However, the purpose of FINDE is to find and preserve general first integrals, including energy and mass not limited by constraints.

\paragraph{Dirac Structure}
A Dirac structure is named after a Dirac bracket, a generalization of the Poisson bracket~\citep{VanderSchaft2014}, and can be found in various systems.
For a rolling disk, the direction in which the disk can move forward without slipping is limited by the disk's orientation.
This constraint is called a non-holonomic constraint.
In an electric circuit, when elements are connected in series, the current flow through each element is always the same.
This constraint is called Kirchhoff's current law.
One can find Dirac structures in these systems.
The dissipative SymODEN was proposed to model a port-Hamiltonian system in the canonical form~\citep{Zhong2020a}, which is a special case of the Dirac structure.
To the best of our knowledge, a neural network model for a general Dirac structure has not yet been proposed.
FINDE is the first neural network method to learn Dirac structures better than NODE can, even though it is not specialized for Dirac structures.

\paragraph{PDE with Mass Conservation}
The total mass of a PDE system is sometimes preserved~\citep{Furihata2010}.
The KdV equation is a Hamiltonian system that describes shallow water waves, in which the energy and total mass are preserved.
The Cahn--Hilliard equation is a model of phase separation of copolymer melts, in which the total mass is preserved, but the energy is dissipated.
In general, a quantity in an area is preserved if its flux entering minus its flux leaving is zero.
Deep conservation extracts latent dynamics of a PDE system and preserves a quantity of interest by forcing its flux to be zero~\citep{Lee2019a}.
HNN++ also ensures mass conservation by designing a coefficient matrix that determines local interaction~\citep{Matsubara2020}.

\paragraph{General First Intergals}
A concurrent study, ``Constants-of-motion network,'' introduced the penalty loss function so that NODEs learn to preserve first integrals~\citep{Kasim2022}; however, unlike other related methods, this method does not guarantee preservation.
A Noether network was proposed to model videos that do not always capture physical phenomena~\citep{Alet2021}.
A subset of the latent variable is assumed to represent image features that do not change during a video, such as the appearance of objects.
For prediction, these features are forced not to change.
The Noether network is potentially useful for learning physical phenomena from videos, but is more similar to semantic manipulation of latent variables~\citep{Shen2020b}.

Some studies have investigated methods that do not predict dynamics but specialize in finding first integrals~\citep{Fukunaga1971,Liu2021b}.
These methods can be used to help FINDE determine the hyperparameter $K$.
They commonly estimate the number $(N-K)$ of dimensions of the tangent space $\TuM'$ of the submanifold $M'$ at point $\vu$ using its neighbors.
For example, AI Poincar\'e proposed by \citet{Liu2021b} assumes that all data points share the submanifold $\M'$ and uses an autoencoder to reconstruct the tangent space $\TuM'$.
Hence, it can only process a single long time series with fixed first integrals.
In contrast, our proposed FINDE can leverage a dataset of multiple time series with different values of the first integrals.

\section{Details of Methods}

\subsection{Derivation of FINDE}\label{appendix:derivation}
\paragraph{Continuous FINDE (cFINDE)}
Let $\vu^{s}$ denote a current state and $\hat f$ denote a vector field.
After a time interval $\Delta t$, the state transitions to $\hat \vu^{s+1}$.
A typical projection method projects the state $\tilde\vu^{s+1}$ onto a submanifold $\M'$ and obtains a state $\vu^{s+1}$, which preserves the first integrals $\vV=(V_1\ \dots\ V_K)^\top$.
This procedure is defined as an optimization problem in Eq.~\eqref{eq:projection_method};
\begin{equation}
  \arg\min_{\vu^{s+1}} \|\vu^{s+1}-\tilde\vu^{s+1}\| \mbox{ subject to } V_k(\vu^{s+1})-V_k(\vu^s)=0 \mbox{ for }k=1,\dots,K.
\end{equation}
One can solve the problem using the method of Lagrange multipliers.
A Lagrangian function is
\begin{equation}
  \textstyle F(\vu^{s+1},\vlambda)=\frac{1}{2}||\vu^{s+1}-\tilde\vu^{s+1}||^2_2+(\vV(\vu^{s+1})-\vV(\vu^s))^\top\vlambda',~\label{eq:lagrangian_function}
\end{equation}
where $\vlambda'$ is the Lagrange multiplier.
The stationary point satisfies
\begin{equation}
  \begin{aligned}
    \textstyle\pderiv{F}{\vu^{s+1}} & \textstyle=\vu^{s+1}-\tilde\vu^{s+1}+\left(\pderiv{\vV}{\vu^{s+1}}\right)^{\!\!\!\top}\vlambda'=\mathbf{0}, \\
    \textstyle\pderiv{F}{\vlambda'} & =\vV(\vu^{s+1})-\vV(\vu^s)=\mathbf{0}.
  \end{aligned}
\end{equation}
Subsequently, a projection method can be redefined as
\begin{equation}
  \begin{aligned}
    \vu^{s+1}                 & \textstyle=\tilde\vu^{s+1}-\left(\pderiv{\vV}{\vu^{s+1}}\right)^{\!\!\!\top}\vlambda', \\
    \vV(\vu^{s+1})-\vV(\vu^s) & =\mathbf{0}.
  \end{aligned}\label{eq:projection_redefined}
\end{equation}
We transform Eq.~\eqref{eq:projection_redefined} into
\begin{equation}
  \begin{aligned}
    \textstyle\frac{\vu^{s+1}-\vu^{s}}{\Delta t}         & \textstyle=\frac{\tilde\vu^{s+1}-\vu^s}{\Delta t}-\left(\pderiv{\vV}{\vu^{s+1}}\right)^{\!\!\!\top}\vlambda, \\
    \textstyle\frac{\vV(\vu^{s+1})-\vV(\vu^s)}{\Delta t} & =\mathbf{0},\label{eq:FINDE_before_limit}
  \end{aligned}
\end{equation}
where $\vlambda=\vlambda'/\Delta t$.
Taking the limit as $\Delta t\rightarrow +0$, we obtain Eq.~\eqref{eq:cFINDE_base};
\begin{equation}
  \begin{aligned}
    f(\vu^{s})                & \textstyle=\hat f(\vu^{s})-\left(\pderiv{\vV}{\vu^{s}}\right)^{\!\!\!\top}\vlambda, \\
    \textstyle\ddt \vV(\vu^s) & =\mathbf{0}.
  \end{aligned}
\end{equation}
The second equation ensures that a state transition following the new vector field $f$ preserves the first integrals $\vV$.
By eliminating the Lagrange multiplier $\vlambda(\vu)$, we define the cFINDE as in Eq.~\eqref{eq:cFINDE_equation}, that is,
\begin{equation}
  \textstyle \ddt\vu=f(\vu)=(I-Y(\vu))\hat f(\vu) \text{ for } Y(\vu)=M(\vu)^\top(M(\vu) M(\vu)^\top)^{-1}M(\vu), \label{eq:cFINDE_equation_appendix}
\end{equation}
where $M=\pderiv{\vV}{\vu}$.
Because of the above derivation, the cFINDE can be considered a continuous-time version of a projection method.
The preservation of first integrals can be proved as follows.

\begin{proof}[Proof of Theorem~\ref{thm:cFINDE}]
  \begin{equation*}
    \begin{aligned}
      \textstyle \ddt \vV(\vu)
       & =\textstyle \pderiv{\vV}{\vu}\ddt\vu                                               \\
       & =\textstyle M(\vu)f(\vu)                                                           \\
       & =\textstyle M(\vu)(I-M(\vu)^\top(M(\vu) M(\vu)^\top)^{-1}M(\vu))\hat f(\vu)        \\
       & =\textstyle (M(\vu)-(M(\vu)M(\vu)^\top)(M(\vu) M(\vu)^\top)^{-1}M(\vu))\hat f(\vu) \\
       & =\textstyle (M(\vu)-M(\vu))\hat f(\vu)                                             \\
       & =\mathbf{0}.
    \end{aligned}
  \end{equation*}
  Hence, it holds that $\ddt V_k(\vu)=0$ for $k=1,\dots,K$, indicating that the cFINDE $\ddt\vu=f(\vu)$ preserves all first integrals $V_k$ in continuous time.
\end{proof}

\paragraph{Discrete FINDE (dFINDE)}
For dFINDE, we take the discrete gradient of the Lagrangian equation in Eq.~\eqref{eq:lagrangian_function} and obtain the discrete version of the necessary conditions for the stationary point;
\begin{equation}
  \begin{aligned}
    \bnabla _{(\vu^{s+1},\vu^{s})}F & \textstyle=\vu^{s+1}-\tilde\vu^{s+1}+\bM(\vu^{s+1},\vu^s)\vlambda'=\mathbf{0}, \\
    \textstyle\pderiv{F}{\vlambda'} & =\vV(\vu^{s+1})-\vV(\vu^s)=\mathbf{0}.
  \end{aligned}
\end{equation}
$\bM(\vu^{s+1},\vu^s)$ corresponds to the Jacobian $\pderiv{\vV}{\vu}$.
By substituting the base model $\frac{\tilde \vu^{s+1}-\vu^{s}}{\Delta t^s}=\hat\psi(\vu^{s};\Delta t^s)$ and the dFINDE $\frac{\vu^{s+1}-\vu^{s}}{\Delta t^s}=\psi(\vu^{s+1},\vu^{s};\Delta t^s)$ into the above equation and dividing the first equation by $\Delta t$, we obtain Eq.~\eqref{eq:psi};
\begin{equation}
  \begin{aligned}
    \textstyle \frac{\vu^{s+1}-\vu^{s}}{\Delta t^s} & =\psi(\vu^{s+1},\vu^{s};\Delta t^s),                                                    \\
    \psi(\vu^{s+1},\vu^{s};\Delta t^s)              & =\hat\psi(\vu^{s};\Delta t^s)-\bM(\vu^{s+1},\vu^{s})^{\top}\vlambda(\vu^{s+1},\vu^{s}), \\
    \vV(\vu^{s+1})-\vV(\vu^{s})                     & =\mathbf{0},
  \end{aligned}\label{eq:psi_appendix}
\end{equation}
where $\vlambda=\vlambda'/\Delta t^s$.
By eliminating the Lagrange multiplier $\vlambda$, we define the dFINDE as in Eq.~\eqref{eq:update_dFINDE}, that is,
\begin{equation}
  \textstyle \frac{\vu^{s+1}-\vu^{s}}{\Delta t^s}=\psi(\vu^{s+1},\vu^{s};\Delta t^s)=(I-\bY(\vu^{s+1}\!\!,\vu^{s}))\hat\psi(\vu^{s};\Delta t^s)
  \text{ for } \bY=\bM^\top(\bM\ \bM^\top)^{-1}\bM.
  \label{eq:update_dFINDE_appendix}
\end{equation}
The preservation of first integrals can be proved as follows.
\begin{proof}[Proof of Theorem~\ref{thm:dFINDE}]
  \begin{equation*}
    \begin{aligned}
      \vV(\vu^{s+1})-\vV(\vu^{s})
       & =\textstyle \bM(\vu^{s+1},\vu^{s}) (\vu^{s+1}-\vu^{s})                                                 \\
       & =\textstyle \bM(\vu^{s+1},\vu^{s}) \psi(\vu^{s+1},\vu^{s};\Delta t^s)\Delta t^s                        \\
       & =\textstyle \bM (I-\bM^\top(\bM\ \bM^\top)^{-1}\bM)\hat\psi(\vu^{s+1},\vu^{s};\Delta t^s)\Delta t^s    \\
       & =\textstyle (\bM-(\bM\bM^\top)(\bM\ \bM^\top)^{-1}\bM)\hat\psi(\vu^{s+1},\vu^{s};\Delta t^s)\Delta t^s \\
       & =\textstyle (\bM-\bM)\hat\psi(\vu^{s+1},\vu^{s};\Delta t^s)\Delta t^s                                  \\
       & =\mathbf{0}.
    \end{aligned}
  \end{equation*}
  Hence, it holds that $V_k(\vu^{s+1})=V_k(\vu^s)$ for $k=1,\dots,K$, indicating that the dFINDE $\frac{\vu^{s+1}-\vu^{s}}{\Delta t^s}=\psi(\vu^{s+1},\vu^{s};\Delta t^s)$ preserves all first integrals $V_k$ in discrete time.
\end{proof}

\subsection{Discrete Gradient}\label{appendix:discrete_gradient}
A discrete gradient is a discrete analogue to a gradient~\citep{Furihata2010,Gonzalez1996,Hong2011}.
Discrete gradients that satisfy Definition \ref{def:discrete_gradient} are not unique, and many variations have been proposed.
For a neural network, \citet{Matsubara2020} proposed the automatic discrete differentiation algorithm (ADDA).
We briefly introduce the algorithm in the case of finite-dimensional Euclidean spaces.
The differential $\dif g$ of a function $g:\R^N\rightarrow\R^M$ is a linear operator $\dif g_\vu:\R^N\rightarrow\R^M$ at point $\vu$ and satisfies
\begin{equation}
  \lim_{||\vh||_{\R^N}\rightarrow 0}\frac{||g(\vu+\vh)-g(\vu)+\dif g_\vu(\vh)||_{\R^M}}{||\vh||_{\R^N}}=0.
\end{equation}
The differential $\dif g$ acting on a vector $\vw$ is equivalent to the product of a vector $\vw$ with the Jacobian $J_g(\vu)$ of the function $g$ at point $\vu$: $\dif g_\vu(\vw)=J_{g(\vu)}\vw$.
Similarly, according to the chain rule, the differential $d(h\circ g)$ of a composition $h\circ g$ of functions $g,h$ is equivalent to the multiplication with a series $J_{h(g(\vu))}J_{g(\vu)}$ of Jacobians.
Therefore, the automatic differentiation algorithm obtains the differential of a neural network.
The differential $\dif g$ of a function $g:\R^N\rightarrow\R$ is a horizontal vector, and the gradient $\nabla g$ of the function $g$ is a vertical vector dual to the differential.
Therefore, the gradient $\nabla g$ is obtained by transposing the differential $\dif g$.
The ADDA replaces each Jacobian with its discrete analogue.
For linear layers, such as fully-connected and convolution layers, the discrete Jacobian is identical to the ordinary Jacobian.
For element-wise nonlinear layers, such as activation functions, a diagonal matrix composed of the slopes between two inputs can act as the discrete Jacobian.
A discrete gradient obtained by the above steps satisfies Definition \ref{def:discrete_gradient}.

\subsection{Prediction and Training Procedures}\label{appendix:training_objective}
For ODEs modeled by neural networks, various training and prediction strategies have been proposed to date~\citep{Chen2018e,Chen2020a,Course2020,Matsubara2020,Zhong2020}; FINDE can adopt any of these.
In our experiments, we used the following simple strategies.

In the case of the cFINDE and base models, taking a state $\vu_\textrm{GT}^{s}$ from the dataset, a numerical integrator solves the ODE $\ddt\vu=f(\vu)$ and predicts the next state $\vu_\textrm{pred.}^{s+1}$.
This process can be informally expressed as
\begin{equation}
  \vu_\textrm{pred.}^{s+1}\simeq\vu_\textrm{GT}^{s}+\int_{t^s}^{t^{s}+\Delta t^s}f(\vu(\tau))\rd\tau \text{ for } \vu(t^s).
\end{equation}
We solved this integration using \textsf{torchdiffeq.}\hspace{0pt}\textsf{odeint}.
The prediction accuracy can be evaluated using the difference between the predicted state $\vu_\textrm{pred.}^{s+1}$ and ground truth $\vu_\textrm{GT}^{s+1}$ taken from the dataset.
We normalized the difference by the time-step size $\Delta t^s$ and defined the \emph{1-step error} $\mathcal L_\textrm{1-step}$ as
\begin{equation}
  \mathcal L_\textrm{1-step}(\vu_\textrm{pred.}^{s+1};\vu_\textrm{GT}^{s+1},\vu_\textrm{GT}^{s},\Delta t^s)=\left\|\frac{\vu_\textrm{GT}^{s+1}-\vu_\textrm{GT}^{s}}{\Delta t^s}-\frac{\vu_\textrm{pred.}^{s+1}-\vu_\textrm{GT}^{s}}{\Delta t^s}\right\|^2_2.\label{eq:step_error}
\end{equation}
The cFINDE and base models were trained to minimize the 1-step error $\mathcal L_\textrm{1-step}$.

In the case of the dFINDE, the next state $\vu_\textrm{pred.}^{s+1}$ is predicted by solving Eq.~\eqref{eq:update_dFINDE} as an implicit scheme; in particular,
\begin{equation}
  \arg\min_{\vu_\textrm{pred.}^{s+1}} \left\|\frac{\vu_\textrm{pred.}^{s+1}-\vu_\textrm{GT}^{s}}{\Delta t^s}-(I-\bY(\vu_\textrm{pred.}^{s+1},\vu_\textrm{GT}^{s}))\hat\psi(\vu_\textrm{GT}^{s};\Delta t^s)\right\|.
\end{equation}
Therefore, prediction by the dFINDE is implicit.
For evaluation, we solved this scheme using \textsf{scipy.}\hspace{0pt}\textsf{optimize.}\hspace{0pt}\textsf{fsolve} and obtained the 1-step error in Eq.~\eqref{eq:step_error}.
However, during the training phase, the ground truth $\vu_\textrm{GT}^{s+1}$ of the next state is known.
Hence, we substituted this into Eq.~\eqref{eq:update_dFINDE}, and then used the difference between the left- and right-hand sides of the dFINDE as the loss function:
\begin{equation}
  \mathcal L_\textrm{dFINDE}(\vu_\textrm{GT}^{s+1},\vu_\textrm{GT}^{s},\Delta t^s)=\left\|\frac{\vu_\textrm{GT}^{s+1}-\vu_\textrm{GT}^{s}}{\Delta t^s}-(I-\bY(\vu_\textrm{GT}^{s+1},\vu_\textrm{GT}^{s}))\hat\psi(\vu_\textrm{GT}^{s};\Delta t^s)\right\|^2_2.
\end{equation}
The discrete Jacobian $\bM$ (and hence $\bY$) can be obtained explicitly, and an explicit numerical integrator can be used for the base model $\hat \psi$.
Hence, the process to obtain the value of the loss function is explicit, and the dFINDE can be trained in an explicit way, whereas the prediction is still implicit.

Some previous studies have proposed alternative strategies.
For example, a loss function can be defined as the sum of the errors at multiple time points during a long-term prediction.
The cFINDE can naturally adopt such a training strategy, and the dFINDE can adopt it after a minor modification.
While it is helpful to pursue absolute performance, it requires additional hyperparameters, such as the length of prediction time, and additional effort to adjust them.
We used the 1-step error in the present study for simplicity and fair comparisons.

The function $V(\vu)$ learning a first integral may become a constant function during training; subsequently, its Jacobian matrix vanishes ($\pderiv{V(\vu)}{\vu}\equiv 0$).
In this case, our algorithm returns a division-by-zero error because it requires the inverse of the matrix $\pderiv{V(\vu)}{\vu}\pderiv{V(\vu)}{\vu}^\top$ for the projection.
We have not taken any special measures to prevent such errors, but no errors occurred in any experiments with proper settings.
The division-by-zero errors have occurred only when FINDE assumes an unreasonable number of first integrals (e.g., $K=6$ for the double pendulum, which has five first integrals).
FINDE works correctly even when the functions $f(\vu)$ and $V(\vu)$ learn the same first integrals; we verified such a case in Section \ref{sec:first_integral_preservation}, where both functions are known.

FINDE learns first integrals point-by-point, and the found first integral is not always consistent over the domain.
The same can be said about the energy function of HNN, and this type of problem is an open problem for neural network models of dynamical systems.

\section{Details of Datasets}\label{appendix:datasets}

To generate each dataset, we used scipy package and the Dormand--Prince method (dopri5) with the default relative tolerance of $10^{-9}$, unless otherwise stated.
Experiments on the KdV dataset were performed with double precision, and all other experiments were performed with single precision.

\paragraph{Hamiltonian System in Canonical Form: Two-Body Problem}
A gravitational two-body problem on a 2-dimensional configuration space has a state $\vu$ composed of the 4-dimensional position $\vq=(x_1\ y_1\ x_2\ y_2)^\top$ and 4-dimensional velocity $\vv=(v_{x1}\ v_{y1}\ v_{x2}\ v_{y2})^\top$.
This is a second-order ODE, indicating that $\ddt \vq=\vv$.
The momentum $p_{x_1}$ of $x_1$ equals $m_1v_{x_1}$.
The time-derivative $\ddt\vv$ of the velocity $\vv$ is called the acceleration.
The acceleration of $x_1$ is given by $\ddt v_{x_1}=-Gm_1m_2\frac{x_1-x_2}{((x_1-x_2)^2+(y_1-y_2)^2)^{3/2}}$, where $G$, $m_1$, and $m_2$ denote the constant of gravity and masses of two bodies, respectively.
The same process applies for the remaining positions.

The total energy of the two-body problem is given by
\begin{equation}
  H=\frac{1}{2}(m_1(v_{x1}^2+v_{y1}^2)+m_2(v_{x2}^2+v_{y2}^2))-\frac{Gm_1m_2}{\sqrt{(x_1-x_2)^2+(y_1-y_2)^2}}.
\end{equation}
The first and second terms denote the kinetic and potential energies, respectively.
The two-body problem is a Hamiltonian system, and the dynamics mentioned above can be rewritten as Hamilton's equation.
The Hamiltonian $H$ is one of the first integrals; the two-body problem has other first integrals, such as the linear momenta in the $x$- and $y$-directions
\begin{equation}
  p_x  =\frac{m_1v_{x1}+m_2v_{x2}}{m_1+m_2},\
  p_y  =\frac{m_1v_{y1}+m_2v_{y2}}{m_1+m_2},
  \label{eq:linear_momentum}
\end{equation}
and angular momentum~\citep{Hairer2006}.

We set $G$, $m_1$, and $m_2$ to 1.0.
The initial distance $r_1=\sqrt{x_1^2+y_1^2}$ of a mass $m_1$ from the origin was set to $r_1\sim\mathcal U(0.5,1.0)$, and the initial angle $\theta_1=\tan^{-1}(\frac{y_1}{x_1})$ was set to $\theta_1\sim\mathcal U(0,2\pi)$.
The initial speed $|v_1|=\sqrt{v_{x_1}^2+v_{y_1}^2}$ was set to $\frac{1}{2r^2}\epsilon_v$, where $\epsilon_v\sim\mathcal N(1,0.05)$.
The initial angle of the velocity was set to $\theta\pm0.5\pi+\epsilon_\theta\pi$, where $\epsilon_\theta\sim\mathcal N(0,0.05)$.
The initial condition of the other mass $m_2$ was set to the opposite of the mass $m_1$.
Subsequently, the two masses trace elliptical orbits, and when $\epsilon_v=\epsilon_\theta=0$, they trace exactly circular orbits.
In addition, we added a perturbation following $\mathcal N(0,0.01)$ to the velocities of both masses, which corresponds to the center-of-gravity velocity.

We set the time-step size $\Delta t$ to 0.01 and generated 1,000 time-series of $S=500$ steps for training and 10 time-series of $S=10,000$ steps for evaluation.
We trained each model for 100,000 iterations.

\paragraph{Hamiltonian System in Non-Canonical Form: KdV equation}
The KdV equation is a model of shallow water waves and is known to have soliton solutions~\citep{Furihata2001}.
The dynamics is given by
\begin{equation}
  u_t = -\alpha uu_x + \beta u_{xxx},
\end{equation}
where $x$ denotes the spatial position and the subscripts denote partial derivatives; for example, $u_t=\pderiv{u}{t}$.
The Hamiltonian is given by
\begin{equation}
  H(u) = \int -\frac{1}{6}\alpha u^3 -\frac{1}{2} \beta u_x^2\ \dif x.\label{eq:kdv_energy}
\end{equation}
As Hamilton's equation $\ddt u=S\nabla H$, the partial differential operator $\pderiv{}{x}$ acts as the coefficient matrix $S$.
This system is Liouville integrable and has infinitely many first integrals, including the Hamiltonian $H$, total mass $I_1=\int u\dif x$, and $T_2=\int u^2\dif x$~\citep{Miura1968}.
Other first integrals are defined using higher-order partial derivatives.

For PDEs, PINNs are known to provide solutions when symbolic equations and boundary conditions are given~\citep{Raissi2019}.
We, in contrast, consider learning spatially discretized PDEs as ODEs from observed data and solving them using numerical integrators, in the same context as NODEs and HNNs; this topic has also been studied extensively~\citep{Long2018a,Matsubara2020,Sun2020,Holl2020}.
Following the experiments in a previous study~\citep{Matsubara2020}, we discretized the KdV equation in space; it no longer has infinitely many first integrals.
We set $\alpha=-6$, $\beta=1$, spatial size to 10 space units, and space mesh size to 0.2; the system state $\vu$ had 50 elements.
We generated two solitons as the initial condition; each was expressed as $-\frac{12}{\alpha} \kappa^2 \mathrm{sech}^2(\kappa (x - d))$, where the size $\kappa$ followed $\mathcal U(0.5,2)$ and the initial position $d$ of one soliton was set to be at least 2.0 from that of the other.

We set the time-step size $\Delta t$ to 0.001 and generated 1,000 time-series of $S=500$ steps for training and 10 time-series of $S=10,000$ steps for evaluation, using the discrete gradient method to ensure energy conservation~\citep{Furihata2001}.
We trained each model for 30,000 iterations.

Due to the spatial discretization, the KdV dataset contains spatial truncation errors.
When the neural network learns this dataset, no spatial truncation errors are additionally introduced.
An evaluation using the analytical solution as a dataset or datasets created with different spatial resolutions is included in future work.

\paragraph{Poisson System: Double Pendulum}
A double pendulum (2-pend) is depicted in Fig.~\ref{fig:2pend_model}.
In polar coordinates, this is a Hamiltonian system.
The state is composed of the angles $(\theta_1,\theta_2)$ of the two rods and their angular velocities $(\omega_1,\omega_2)$.
This is also a second-order ODE, indicating that $\ddt\theta_1=\omega_1$ and $\ddt\theta_2=\omega_2$.

Let $l_1,l_2$ denote the lengths of the two rods, $m_1,m_2$ denote the masses of the two weights, and $g$ denote the gravitational acceleration.
The acceleration is given by
\begin{equation}
  \begin{aligned}
    \ddt\omega_1 & = \frac{m_2 g \sin \theta_2 \cos \Delta - (l_1 \omega_1^2 \cos \Delta + l_2 \omega_2^2) m_2 \sin \Delta - (m_1 + m_2) g \sin \theta_1}
    {l_1 (m_1 + m_2 \sin^2 \Delta)},                                                                                                                              \\
    \ddt\omega_2 & = \frac{(m_1 + m_2) (l_1 \omega_1^2 \sin \Delta - g \sin \theta_2 + g \sin \theta_1 \cos \Delta) + m_2 l_2 \omega_2^2 \sin \Delta \cos \Delta}
    {l_2 (m_1 + m_2 \sin^2 \Delta)},
  \end{aligned}
\end{equation}
where $\Delta=\theta_1-\theta_2$.
In 2-dimensional Cartesian coordinates, the state is composed of the positions $(x_1,y_1,x_2,y_2)$ of the two masses and the corresponding velocities $(v_{x1},v_{y1},v_{x2},v_{y2})$.
The position is transformed by $x_1=l_1\sin \theta_1$, $y_1=l_1\cos \theta_1$, $x_2=x_1+l_2\sin \theta_2$, and $y_2=y_1+l_2\cos \theta_2$, and the velocity is transformed accordingly.
The total energy $H$ is given by
\begin{equation}
  H=\frac{1}{2}(m_1 (v_{x_1}^2+v_{y_1}^2)+m_2 (v_{x_2}^2+v_{y_2}^2))+g(m_1 y_1+m_2 y_2).
\end{equation}
\begin{wrapfigure}{r}{1.1in}
  \vspace*{-4mm}
  \centering
  \includegraphics[scale=0.5,page=1]{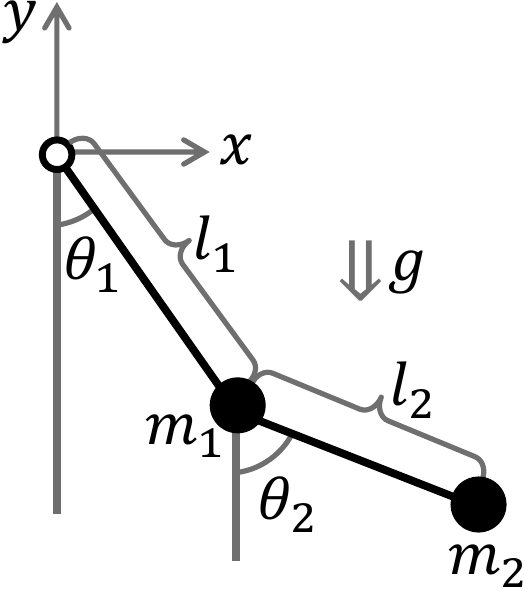}
  \vspace*{-2mm}
  \caption{Diagram of the double pendulum.}\label{fig:2pend_model}
  \vspace*{-5mm}
\end{wrapfigure}
The first and second terms denote the kinetic and potential energies, respectively.
The double pendulum is no longer a Hamiltonian system in Cartesian coordinates.
Because the lengths of the two rods are constant, the double pendulum has two constraints on the position: $l_1^2=x_1^2+y_1^2$ and $l_2^2=(x_2-x_1)^2+(y_2-y_1)^2$.
These constraints are holonomic constraints, and they lead to constraints involving the velocity, namely $0=x_1v_{x_1}+y_1v_{y_1}$ and $0=(x_2-x_1)(v_{x_2}-v_{x_1})+(y_2-y_1)(v_{y_2}-v_{y_1})$.
When the constraints involving the velocity are satisfied, the holonomic constraints are implicitly satisfied.
Therefore, the number of first integrals is five; however, three first integrals are sufficient to determine the dynamics.
The dynamics is degenerate and classified as a constrained Hamiltonian system, or a Poisson system in a more general case.

We set the masses of the two weights to $m_1=m_2=1.0$ and the gravitational acceleration $g$ to $9.8$.
We set the lengths $l_1,l_2$ of the two rods to follow $\mathcal U(0.9,1.1)$, the initial angles $\theta_1,\theta_2$ to follow $\mathcal U(-0.5,0.5)$, and the initial angular velocities $\dot\theta_1,\dot\theta_2$ to follow $\mathcal U(-0.1,0.1)$.

We set the time-step size $\Delta t$ to 0.1 and generated 1,000 time-series of $S=500$ steps for training and 10 time-series of $S=5,000$ steps for evaluation.
We trained each model for 100,000 iterations.

\paragraph{Dirac Structure: FitzHugh--Nagumo Model}
R.~FitzHugh proposed a model of the electrical dynamics of a biological neuron, and J.~Nagumo created an equivalent electric circuit.
This model is called the FitzHugh--Nagumo model~\citep{Izhikevich2006e} and is a modified version of the van der Pol oscillator; the state oscillates when the magnitude of the external current source $I$ is within an appropriate range.
The circuit comprises a resistor $R$, inductor $L$, capacitor $C$, tunnel diode $D$, and voltage source $E$ connected as shown in Fig.~\ref{fig:FitzHughNagumo}.
The whole circuit is connected to an external current source $I$.
Let $I_R$ denote the current through the resistor $R$, and $V_R$ denote the applied voltage.
Ohm's law and other properties of the elements give $V_R=I_R R$, $C\ddt V_C=I_C$, $L\ddt I_L=V_L$, and $I_D=D(V_D)$, where we treat $D$ as a nonlinear function.
Kirchhoff's current law (KCL) obtains $I_C+I_D+I_R=I$ and $I_R=I_L$, and Kirchhoff's voltage law (KVL) obtains $V_C=V_D=V_R+V_L+E$.
We denote $W=I_R$ and $V=V_C$, and set $L=1/0.08$, $R=0.8$, $C=1.0$, $V_E=-0.7$, and $D(V)=V^3/3-V$.
Subsequently, we obtain the FitzHugh--Nagumo model of the original parameters as
\begin{equation}
  \begin{aligned}
    \ddt V & =V-V^3/3-W+I,      \\
    \ddt W & =0.08(V+0.7-0.8W).
  \end{aligned}
\end{equation}
Due to the resistor $R$, the FitzHugh--Nagumo model is not an energy-conserving system.

\begin{wrapfigure}{r}{1.6in}
  \vspace*{-2mm}
  \includegraphics[scale=0.5,page=2]{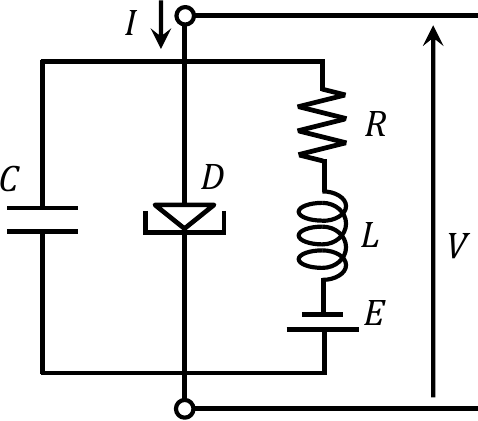}
  \caption{Circuit diagram of FitzHugh--Nagumo model~\citep{Izhikevich2006e}.}\label{fig:FitzHughNagumo}
\end{wrapfigure}
Consider a situation where the current through and the voltage applied to stateful elements (capacitors and inductors) are measurable, but the connections between the elements are unknown.
We treated $I_C,I_L,V_C,V_L$ as the system state $\vu$.
Because the state is in 4-dimensional space and the dynamics is intrinsically 2-dimensional, there exist two first integrals; for example, but not limited to, $I=I_C+D(V_C)+I_L$ and $E=V_C-I_LR-V_L$.
This type of electric circuit is an example of a Dirac structure because the state variables are constrained by the circuit topology and Kirchhoff's current and voltage laws~\citep{VanderSchaft2014}.
From the viewpoint of generalized Hamiltonian systems, $(I_L,V_C)$ corresponds to the position, and $(V_L,I_C)$ corresponds to the momentum.
The electric circuit can be described as a port-Hamiltonian system in a non-canonical form.
Because of the non-canonical form, the FitzHugh--Nagumo model is outside the scope of CHNN and dissipative SymODEN~\citep{Finzi2020b,Zhong2020a}.

We set the external current source $I$ to follow $\mathcal U(0.7,1.1)$, set the initial values of $V$ and $W$ to follow $\mathcal U(-1.5,1.5)$ and $\mathcal U(0.0,2.0)$, and transformed them to the state.

We set the time-step size $\Delta t$ to 0.1 and generated 1,000 time-series of $S=500$ steps for training and 10 time-series of $S=2,000$ steps for evaluation.
We trained each model for 30,000 iterations.


\section{Additional Results and Discussion}\label{appendix:discussion}
\subsection{Demonstration of First Integral Preservation}\label{appendix:mass_spring_dopri}
\begin{wrapfigure}{r}{1.7in}
  \vspace*{-15mm}
  \includegraphics[scale=0.9]{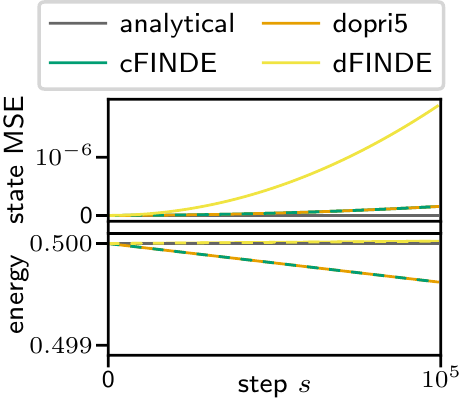}
  \vspace*{-2mm}
  \caption{Integration of a known mass-spring system by Dormand--Prince integrator.
    (top) Mean squared errors in states predicted by comparison methods.
    (bottom) Energy calculated from the states predicted.
  }\label{fig:mass_spring_dopri}
  \vspace*{-3mm}
\end{wrapfigure}
In Fig.~\ref{fig:mass_spring}, we examined a mass-spring system and FINDE using the leapfrog integrator.
We also examined the case with the Dormand--Prince integrator (dopri5), as shown in Fig.~\ref{fig:mass_spring_dopri}.
We increased the number of steps to $10^5$, and displayed the MSEs of the state instead of the state itself.
First, we focus on the energy.
Even using the Dormand--Prince integrator, a fourth-order method, the energy is slightly decreased.
The cFINDE with the Dormand--Prince integrator shows the same tendency.
This phenomenon is due to temporal discretization errors and is called energy drift.
The dFINDE with the Dormand--Prince integrator significantly suppresses the error in energy.
The remaining error is caused by rounding errors.

When the focus is on the MSEs of the state, the trend is different: the dFINDE with the Dormand--Prince integrator suffers from the most significant errors in state.
Although the dFINDE is designed to eliminate temporal discretization errors in energy, it does not necessarily reduce those in state.
In contrast, the Dormand--Prince integrator is designed to suppress temporal discretization errors in state.

Therefore, there is no guarantee that the dFINDE improves the prediction performance when defined using errors in state.
Conversely, the experimental results in Table~\ref{tab:performance} demonstrate that the dFINDE is superior to the base model and cFINDE in VPT.
This is because dFINDE reduces the modeling errors rather.
For the mass-spring system, the governing equation is already known as an ODE and is discretized by the dFINDE, leading to temporal discretization errors.
However, when dFINDE learns dynamics from data, the training data points are already sampled in discrete time, and the dFINDE predicts future states in discrete time.
Therefore, no temporal discretization error occurs, and we obtain only the advantages of exactly preserving the first integral.

This type of paradox has been repeatedly discovered in previous studies.
For example, the leapfrog integrator and discrete gradient method are second-order methods.
However, they are superior to the Dormand--Prince integrator when combined with neural networks and learning dynamics from data~\citep{Matsubara2020}.
For better learning (i.e., smaller modeling errors), the preservation of specific properties of target systems is more important than the order of accuracy.

\subsection{Symbolic Regression of Found First Integrals}\label{appendix:symbolic_regression}
Using gplearn (based on genetic programming), we performed a symbolic regression of the first integrals $\vV$ found by the neural network.
We prepared addition, subtraction, multiplication, and division as candidate operations, used Pearson's correlation coefficient as the evaluation criterion, set the early stopping threshold to 0.9, and set the population size to 10,000.
We set the other hyperparameters to their default values, e.g., the maximum number of generations was 20.

We summarize the regression results of the HNN with cFINDE for $K=2$ trained using the two-body dataset in Table~\ref{tab:symbolic}.
Note that Pearson's correlation coefficient is invariant to biases and scale factors.
FINDE is also invariant because it only uses the directions of the gradients of first integrals.
Hence, we removed biases and scale factors from the regression results.
When the focus is on the symbolic regression of the training data, $V_1$, $V_1$, $V_2$, and $V_2$ for trials 0, 1, 2, and 3 are identical to the linear momentum in the $x$-direction up to scale factors; recall that we set $m_1=m_2=1.0$ and see Eq.~\eqref{eq:linear_momentum}.
$V_2$, $V_2$, $V_1$, and $V_1$ for trials 0, 1, 2, and 3 are also identical to the linear momentum in the $y$-direction.
$V_1$ and $V_2$ for trial 4 are weighted sums of the linear momenta in the $x$- and $y$-directions; in particular, they can be regarded as the linear momenta in the $(1,-1)$- and $(1,1)$-directions, respectively.

When the quantities $V_1(\vu)$ and $V_2(\vu)$ are first integrals, any function of only $V_1(\vu)$, $V_2(\vu)$, and arbitrary constants is a first integral functionally dependent on $V_1(\vu)$ and $V_2(\vu)$.
Thus, it is in principle impossible to re-discover a first integral as a well-known symbolic expression, and a failure in symbolic regression is not a problem in any way.
Previous studies introduced certain constraints (such as ``gauge fixing'') for symbolic regression~\citep{Liu2021b}; a combination of such method may improve the results.
However, recent studies on neural networks have revealed that typical initialization and training procedures tend to learn simple functions~\citep{Barrett2021,Cao2021a}.
Additionally, the symbolic regression limited the depth of the computation graph, biasing the results toward simple functions; hence, the found first integrals were identical to the well-known forms and were separated in the $x$- and $y$-directions in most cases.

The same is true for the symbolic regression of the test data, except for $V_1$ for trial 0, which had a small perturbation $\alpha$.
Because of the limited extrapolation ability, neural networks cannot always accurately represent functions outside the training data range.
Once first integrals are found by FINDE and identified as equations by symbolic regression, one can use the equations instead of neural networks, ensuring the preservation of first integrals in the entire domain.
From these results, we can conclude that cFINDE identified the linear momenta.

The state of the KdV dataset has 50 elements, which is too large to apply a symbolic regression.
For the 2-pend and FitzHugh--Nagumo datasets, we did not find consistent equations of first integrals.
For example, the symbolic regression identified a quantity $x_1^2 - y_1$ as a first integral in the 2-pend dataset, which is not directly related to well-known first integrals.
When the angle $\theta_1$ of the upper rod is small, $y_1$ takes a value close to $-1$, and the quantity $x_1^2 - y_1$ is close to $x_1^2+y_1^2$, which is a well-known first integral, namely the square $l_1^2$ of the upper rod length $l_1$.
It is difficult to determine whether this inaccuracy is because of the training of FINDE or symbolic regression.
There may still be room for improvement in the training of FINDE or symbolic regression.

\begin{table}[t]
  \fontsize{8pt}{9pt}\selectfont
  \caption{Symbolic Regression of First Integrals Found in Two-Body Problem}
  \label{tab:symbolic}
  \centering
  \begin{tabular}{ccccc}
    \toprule
                   & \multicolumn{2}{c}{\textbf{Training Data}} & \multicolumn{2}{c}{\textbf{Test Data}}                                                                                        \\
    \cmidrule(lr){2-3}\cmidrule(lr){4-5}
    \textbf{Trial} & \textbf{$V_1$}                             & \textbf{$V_2$}                            & \textbf{$V_1$}                        & \textbf{$V_2$}                            \\
    \midrule
    0              & $v_{x1}\!+\!v_{x2}$                        & $v_{y1}\!+\!v_{y2}$                       & $v_{x1}\!+\!v_{x2}\!+\!\alpha$        & $v_{y1}\!+\!v_{y2}$                       \\
    1              & $v_{x1}\!+\!v_{x2}$                        & $v_{y1}\!+\!v_{y2}$                       & $v_{x1}\!+\!v_{x2}$                   & $v_{y1}\!+\!v_{y2}$                       \\
    2              & $v_{y1}\!+\!v_{y2}$                        & $v_{x1}\!+\!v_{x2}$                       & $v_{y1}\!+\!v_{y2}$                   & $v_{x1}\!+\!v_{x2}$                       \\
    3              & $v_{y1}\!+\!v_{y2}$                        & $v_{x1}\!+\!v_{x2}$                       & $v_{y1}\!+\!v_{y2}$                   & $v_{x1}\!+\!v_{x2}$                       \\
    4              & $v_{x1}\!+\!v_{x2} - v_{y1} - v_{y2}$      & $v_{x1}\!+\!v_{x2}\!+\!v_{y1}\!+\!v_{y2}$ & $v_{x1}\!+\!v_{x2} - v_{y1} - v_{y2}$ & $v_{x1}\!+\!v_{x2}\!+\!v_{y1}\!+\!v_{y2}$ \\
    \bottomrule
  \end{tabular}\\
  \raggedright
  We removed biases and scale factors. $\alpha=0.003(y_1 + y_2)(v_{x2} + x_1 + y_1(v_{x2} + y_1 + y_2) + 1.402)$.
\end{table}

\subsection{Comparison with Model of Known Holonomic Constraints}\label{appendix:holonomic_constraints}
The double pendulum (2-pend) is classified as a constrained Hamiltonian system.
CHNN was proposed for cases when holonomic constraints are known~\citep{Finzi2020b}.
We evaluated comparison methods under the assumption that the holonomic constraints were known.
We summarized the results in Table~\ref{tab:holonomic}.
The HNN, without constraints, completely failed to learn the dynamics.
This is unsurprising because the dynamics of the double pendulum is outside the scope of the HNN.
The two known holonomic constraints lead to two constraints involving the velocity; the CHNN took into account all four known constraints and worked remarkably.
The HNN with cFINDE was given all four known constraints as the first integrals, but did not work properly.
The original purpose of projection methods is to eliminate temporal discretization errors of first integrals but not to change the class to which the dynamics belong.
Therefore, when a target system is not a subject of the base model, the base model with FINDE does not work.
The NODE learns an ODE in a general way, and thus constrained Hamiltonian systems are included in its subjects.
Given all four known constraints, the NODE with cFINDE worked better but never surpassed the CHNN.

However, the CHNN works only for Hamiltonian systems in the canonical form with holonomic constraints.
We also evaluated comparison methods using the 2-body dataset under the assumption that the linear momenta were known as first integrals.
The CHNN attempted to obtain the inverse of a singular matrix and could not learn the dynamics.
In contrast, the cFINDE improved the performances of both NODE and HNN.

Existing methods (e.g., HNN and CHNN) assume geometric structures (e.g., Hamiltonian structure) described in Appendix \ref{appendix:system} in order to guarantee conservation laws.
When multiple structures are assumed at the same time, they must be integrated using appropriate prior knowledge.
If it is possible, it would achieve extremely high performance.
Otherwise, the geometric structures would conflict with each other and would not produce an appropriate model.
This is the reason why CHNN failed to learn the 2-body dataset and HNN+FINDE failed to learn the 2-pend dataset.
In contrast, NODE+FINDE does not assume any geometric structure and assumes first integrals in the most general way, being available to any situation.
Hence, FINDE can assume one or more first integrals without changing anything.

When the detailed properties of target systems are known, one can choose the best models.
If the chosen model is inappropriate, the training procedure totally fails.
FINDE provides a better alternative when prior knowledge is limited.
Moreover, a constrained Hamiltonian system can have first integrals other than holonomic constraints and the Hamiltonian.
In this case, the CHNN with FINDE is potentially the best choice.

\begin{table}[t]
  \fontsize{8pt}{9pt}\selectfont
  \centering
  \caption{Results with Known Holonomic Constraints.}
  \label{tab:holonomic}
  \begin{tabular}{lrlrl}
    \toprule
                              & \multicolumn{2}{c}{\textbf{2-pend}}             & \multicolumn{2}{c}{\textbf{2-body}}                                                                                                       \\
    \cmidrule(lr){2-3}\cmidrule(lr){4-5}
    Model                     & \multicolumn{1}{c}{\textbf{1-step}$\downarrow$} & \multicolumn{1}{c}{\textbf{VPT}$\uparrow$} & \multicolumn{1}{c}{\textbf{1-step}$\downarrow$} & \multicolumn{1}{c}{\textbf{VPT}$\uparrow$} \\
    \midrule
    NODE                      & 0.82\std{0.02\zz}                               & 0.110\std{0.035}                           & 144.21\std{12.65}                               & 0.134\std{0.014}                           \\
    HNN~\citep{Greydanus2019} & 6220.26\std{91.57}                              & 0.002\std{0.000}                           & 5.17\std{0.57\zz}                               & 0.362\std{0.026}                           \\
    CHNN~\citep{Finzi2020b}   & 0.07\std{0.00\zz}                               & 0.928\std{0.036}                           & \multicolumn{2}{c}{(not working)}                                                            \\
    \midrule
    NODE+cFINDE               & 0.71\std{0.04\zz}                               & 0.461\std{0.071}                           & 163.64\std{9.79\zz}                             & 0.147\std{0.024}                           \\
    HNN+cFINDE                & 236.51\std{7.15\zz}                             & 0.020\std{0.002}                           & 8.32\std{0.43\zz}                               & 0.476\std{0.040}                           \\
    \bottomrule
  \end{tabular}
\end{table}

\subsection{Reason for High Performance and How to Determine Number of First Integrals}\label{appendix:how_to_k}
The theoretical explanation for the high performance of neural networks (e.g., HNN) that assume first integrals for physical phenomena is an open question.
\cite{Sannai2021} has theoretically shown that neural networks (e.g., CNNs and GNNs) with symmetry have faster learning convergence, and we consider this approach can be applied to the above question.
At least for cFINDE and dFinde, we have an intuitive but not rigorous explanation; assuming one more first integral (i.e., increasing $K$ by 1) reduces the number of degrees of freedom in the dynamics by 1, narrows the hypothesis space, accelerates learning convergence, and suppresses generalization errors.

As shown in Table~\ref{tab:performance}, the performance of cFINDE and dFINDE is sensitive to the assumed number $K$ of first integrals.
Because $K$ is a hyperparameter, it is basically a subject to be adjusted through evaluations on a validation set.
With inappropriately large $K$, both cFINDE and dFINDE dropped their performance significantly.
See the results of the 2-pend and FitzHugh--Nagumo datasets for $K=6$ and $K=3$, respectively.

However, the performance drop can be found even with the training set.
Table~\ref{tab:performance_train} summarizes the prediction performance on the training set of the 2-pend dataset.
As was the case with the test set, the performance significantly dropped at $K=6$.
This is because NODE with cFINDE for $K=6$ assumes the submanifold $\M'$ to be 2-dimensional.
The submanifold $\M'$ is in fact 3-dimensional, so NODE with cFINDE for $K=6$ is incapable of learning the dynamics and performs poorly even on the training set.
Hence, the training set is enough to avoid a fatally inappropriate $K$.

Alternatively, $K$ can be determined by using other methods (e.g., \citet{Fukunaga1971,Liu2021b}).
Although these methods have some drawbacks introduced in Appendix~\ref{appendix:system}, they may be complementary to FINDE.

\begin{table}[t]
  \centering
  \fontsize{8pt}{9pt}\selectfont
  \captionof{table}{Results of NODE with cFINDE on Training Set of 2-Pend Dataset.}
  \label{tab:performance_train}
  \begin{tabular}{lcrl}
    \toprule
             &     & \multicolumn{2}{c}{\textbf{2-pend}}
    \\
    \cmidrule(lr){3-4}
    Model    & $K$ & \multicolumn{1}{c}{\textbf{1-step}$\downarrow$} & \multicolumn{1}{c}{\textbf{VPT}$\uparrow$} \\
    \midrule
    NODE     & --  & 0.76\std{0.02}                                  & 0.966\std{0.007}                           \\
    \midrule
             & 1   & 0.72\std{0.06}                                  & 0.974\std{0.004}                           \\
             & 2   & 0.69\std{0.08}                                  & 0.981\std{0.014}                           \\
    + cFINDE & 3   & 0.63\std{0.02}                                  & 0.994\std{0.002}                           \\
             & 4   & 0.67\std{0.05}                                  & 0.990\std{0.005}                           \\
             & 5   & 0.65\std{0.02}                                  & 0.998\std{0.000}                           \\
             & 6   & 9.93\std{0.00}                                  & 0.126\std{0.000}                           \\
    \bottomrule
  \end{tabular}
\end{table}

\subsection{Comparison with Modified Neural Projection Method}\label{appendix:neural_projection}
The neural projection method (NPM) also employs a projection method~\citep{Yang2020a}.
Using a manner similar to Newton's method, it enforces the constraint $C(\vu)=0$ by the projection of the state $\vu$ under the assumption that the quantity $C(\vu)$ is always zero.
This assumption holds for some cases (e.g., holonomic constraints in a fixed environment), but not for most first integrals, whose values depend on initial conditions.

For example, the linear momentum in the $x$-direction of the two-body problem is the first integral expressed as $V(\vu)=m_1 v_{x1}(t)+m_2 v_{x2}(t)$.
This quantity $V$ is constant within a trial (i.e., $V(\vu(t))=V(\vu(0))$) and varies between trials depending on the initial speed $ v_{x1}(0)$ and $ v_{x2}(0)$.
The total energy, the total mass, and many other first integrals depend on the initial condition in the same manner; hence, they are outside the scope of the NPM.
In contrast, by imposing the constraint on the gradient $\nabla V=0$ or discrete gradient $\overline\nabla V=0$, our proposed FINDE keeps the quantity $V$ constant and can handle any first integrals.

For comparison, we replaced the constraint $C(\vu)=0$ with $C(\vu^{s+1},\vu^s)=V(\vu^{s+1})-V(\vu^s)=0$ and adopted the NPM to first integrals varying from trial to trial.
We evaluated the modified NPM using the 2-pend dataset.
Because the modified NPM is a discrete-time projection method, we compared it with the discrete-time version of the proposed FINDE (dFINDE).
The results are summarized in Table~\ref{tab:npm}.

\begin{table}[t]
  \fontsize{8pt}{9pt}\selectfont
  \caption{Comparison with Neural Projection Method (NPM)}
  \label{tab:npm}
  \centering
  \begin{tabular}{lccccc}
    \toprule
    \textbf{K} & \multicolumn{2}{c}{\textbf{dFINDE (proposed)}}  & \multicolumn{3}{c}{\textbf{modified NPM}}                                                                                                                       \\
    \cmidrule(lr){2-3}\cmidrule{4-6}
               & \multicolumn{1}{c}{\textbf{1-step}$\downarrow$} & \multicolumn{1}{c}{\textbf{VPT}$\uparrow$} & \multicolumn{1}{c}{\textbf{1-step}$\downarrow$} & \multicolumn{1}{c}{\textbf{VPT}$\uparrow$} & \textbf{successful} \\
    \midrule
    1          & 0.75\std{0.10}                                  & 0.152\std{0.017}                           & 0.73\std{0.08}                                  & 0.150\std{0.014}                           & 5/5                 \\
    2          & 0.74\std{0.05}                                  & 0.271\std{0.111}                           & ---                                             & ---                                        & 0/5                 \\
    3          & 0.69\std{0.05}                                  & 0.447\std{0.081}                           & (0.69\std{0.00})                                & (0.138\std{0.000})                         & 1/5                 \\
    4          & 0.71\std{0.03}                                  & 0.454\std{0.060}                           & (0.72\std{0.03})                                & (0.383\std{0.023})                         & 3/5                 \\
    5          & 0.86\std{0.09}                                  & 0.591\std{0.087}                           & 0.85\std{0.11}                                  & 0.364\std{0.134}                           & 5/5                 \\
    6          & 58.88\std{22.98}                                & 0.037\std{0.039}                           & (1.29\std{0.20})                                & (0.103\std{0.016})                         & 3/5                 \\
    \bottomrule
  \end{tabular}
\end{table}

The dFINDE successfully learned the dynamics in all trials, but the modified NPM failed to learn the dynamics in half the trials (see the rightmost column for the numbers of successful trials out of 5).
The modified NPM often encountered of the underflow of the time-step size or a division by the zero gradient of the first integral.
Even when the learning was successful, the performance of the NPM was inferior to that of the dFINDE.
The modified NPM solved the optimization problem in Eq.~\eqref{eq:projection_method} at every step, but it sometimes diverged or failed to converge, especially in the early phase of learning.
The NPM was successful for fixed environments but might be unsuited for general first integrals varying from trial to trial.
However, the dFINDE does not require solving an optimization problem during training, making the learning process robust against randomness such as initialization.

\end{document}